\documentclass[a4paper]{llncs}

\usepackage{alltt}
\usepackage{caption}
\usepackage{paralist}
\usepackage{nicefrac}
\usepackage{booktabs}
\usepackage{algorithm}
\usepackage{algorithmic}
\usepackage{xspace}
\usepackage{amsmath}
\usepackage{marvosym}
\usepackage{wasysym}
\usepackage{verbatim}
\usepackage{subfigure}
\usepackage{hyperref}
\usepackage{multirow} 
\usepackage{cite}
\usepackage[capitalize]{cleveref}
\usepackage[per-mode=symbol,detect-all]{siunitx}
\usepackage{graphics}   
\usepackage{amssymb}

\usepackage[firstpage]{draftwatermark}\SetWatermarkText{\hspace*{5.2in}\raisebox{5in}{\includegraphics[scale=0.1]{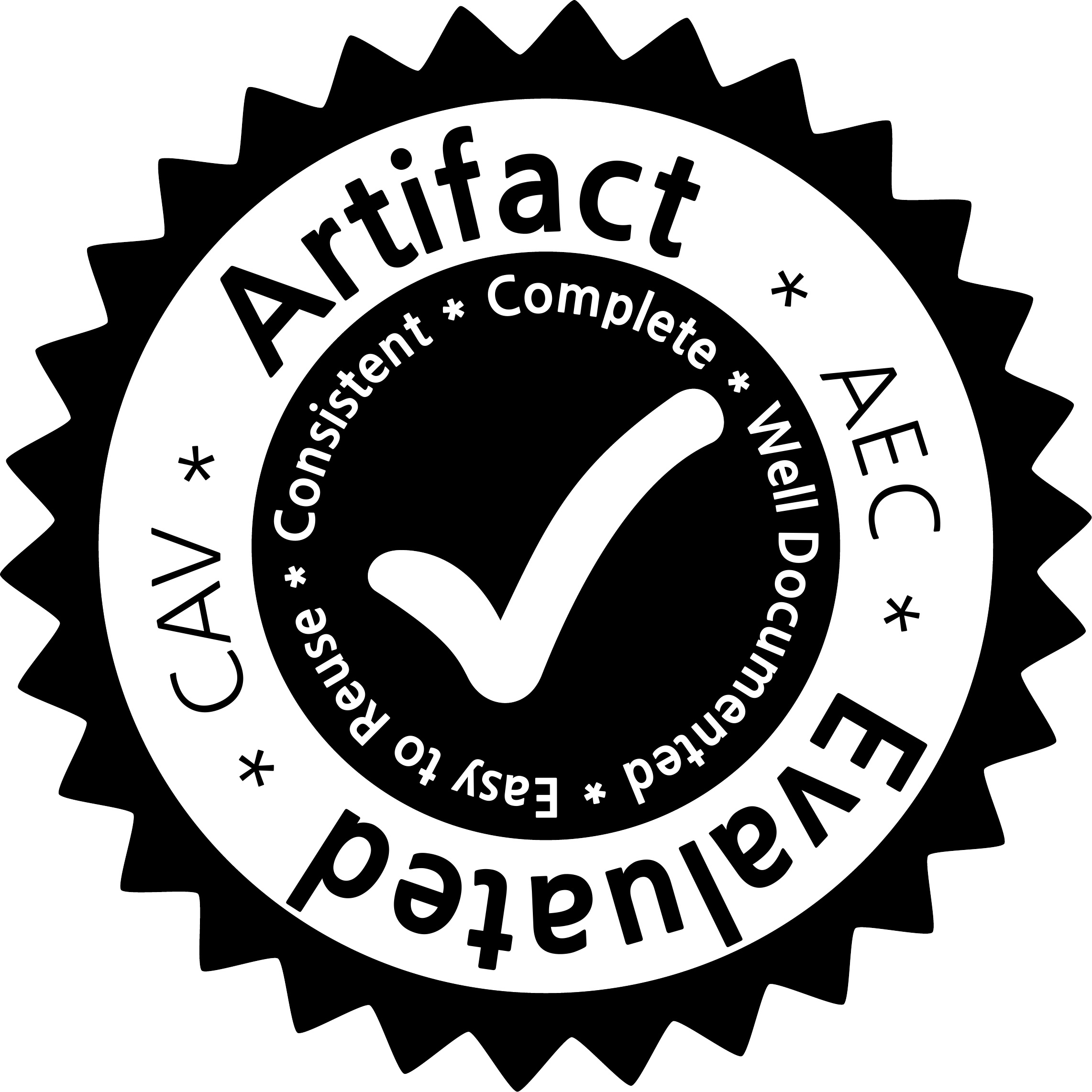}}}\SetWatermarkAngle{0}

\usepackage{tikz,ifthen,pgfplots}
\usetikzlibrary{arrows,trees,backgrounds,automata,shapes,decorations,plotmarks,fit,calc,positioning,shadows,chains}
\tikzstyle{every pin edge}=[<-,shorten <=1pt]
\tikzstyle{neuron}=[circle,fill=black!25,minimum size=17pt,inner sep=0pt]
\tikzstyle{input neuron}=[neuron, fill=green!50]
\tikzstyle{output neuron}=[neuron, fill=red!50]
\tikzstyle{hidden neuron}=[neuron, fill=blue!50]
\tikzstyle{annot} = [text width=4em, text centered]
\usetikzlibrary{calc}
 
\usetikzlibrary{angles,quotes}
\usepgfplotslibrary{groupplots}

\newcommand{\relu}{\text{ReLU}\xspace{}}
\newcommand{\Mo}{\mathbf{I}}

\newcommand{\basic}{\mathcal{B}}

\newcommand{\allvars}{\mathcal{X}}
\newcommand{\ub}{u}
\newcommand{\lb}{l}
\newcommand{\reluSet}{R}
\newcommand{\assignment}{\alpha{}}

\newcommand{\tr}{\mathcal{T}_{\mathbb{R}}}
\newcommand{\trr}{\mathcal{T}_{\mathbb{R}R}}

\newcommand{\sat}{\texttt{SAT}}
\newcommand{\unsat}{\texttt{UNSAT}}
\newcommand{\timeout}{\texttt{TIMEOUT}}

\newcommand{\posWeights}{\text{pos}\xspace{}}
\newcommand{\negWeights}{\text{neg}\xspace{}}

\newcommand{\slackPlus}{\text{slack}^+\xspace{}}
\newcommand{\slackMinus}{\text{slack}^-\xspace{}}
\newcommand{\slack}{\text{slack}\xspace{}}

\newcommand{\drule}[2]{
\renewcommand{\arraystretch}{1.2}
\(\begin{array}{c}
#1 \\
\hline
#2
\end{array}\)
}

\newcommand{\rulename}[1]{\ensuremath{\mathsf{#1}}\xspace}
\newcommand{\irulename}[2]{\ensuremath{\mathsf{#1}_{#2}}\xspace}
\newcommand{\pivot}[1]{\irulename{Pivot}{#1}}
\newcommand{\failure}{\rulename{Failure}}
\newcommand{\update}{\rulename{Update}}
\newcommand{\updateb}{\irulename{Update}{b}}
\newcommand{\updatef}{\irulename{Update}{f}}
\newcommand{\success}{\rulename{Success}}
\newcommand{\reluSuccess}{\rulename{ReluSuccess}}
\newcommand{\pivotForRelu}{\rulename{PivotForRelu}}
\newcommand{\reluSplit}{\rulename{ReluSplit}}

\newcommand{\learnUB}{\rulename{deriveUpperBound}}
\newcommand{\learnLB}{\rulename{deriveLowerBound}}

\newcommand{\pivotOperation}{\textit{pivot}}
\newcommand{\updateOperation}{\textit{update}}

\begin{document}

\title{Reluplex: An Efficient SMT Solver for Verifying 
Deep Neural Networks\thanks{This is the extended version of a paper
  with the same title that appeared at CAV 2017.}}
 
\author{Guy Katz, Clark Barrett, David Dill, Kyle Julian and Mykel Kochenderfer}
\institute{
Stanford University, USA \\
\{guyk, clarkbarrett, dill, kjulian3, mykel\}@stanford.edu
}

\maketitle

\begin{abstract}
Deep neural networks have emerged as a widely used and effective means for
tackling complex, real-world problems. However, a major obstacle in applying them to
safety-critical systems is the great difficulty in providing
formal guarantees about their behavior.
We present a novel, scalable, and efficient technique for verifying 
properties of deep neural networks (or providing counter-examples). The
technique is based on the simplex method, extended to handle
the non-convex \emph{Rectified Linear Unit} (\emph{ReLU})
activation function, which is a crucial ingredient in many modern neural networks.
The verification procedure tackles neural networks as a whole,
without making any simplifying assumptions.
We evaluated our technique on a prototype deep neural
network implementation of the next-generation airborne collision
avoidance system for unmanned aircraft
(ACAS Xu). Results show that our technique can
successfully prove properties of networks that are an order of magnitude larger
than the largest networks verified using existing methods.
\end{abstract}

\section{Introduction}
Artificial neural networks~\cite{RiTo99,FoBeCu16}
have emerged as a
promising approach for creating scalable and robust systems.
Applications include speech
recognition~\cite{HiDeYuDaMoJaSeVaNgSaKi12}, image
classification~\cite{KrSuHi12}, game
playing~\cite{SiHuMaGuSiVaScAnPaLaDi16}, and many others.
It is now clear that software that may be extremely
difficult for humans to implement can instead be created by
training \emph{deep neural networks} (\emph{DNN}s), and that the
performance of these DNNs is often comparable to, or even surpasses, the
performance of manually crafted software. 
DNNs are becoming widespread, and this trend is likely to continue and intensify. 

Great effort is now being put into
using DNNs as controllers for safety-critical systems such as autonomous
vehicles~\cite{BoDeDwFiFlGoJaMoMuZhZhZhZi16} and
airborne collision avoidance systems for unmanned aircraft
(ACAS Xu)~\cite{JuLoBrOwKo16}. 
DNNs are trained over a finite set of inputs and outputs and are
expected to \emph{generalize}, i.e. to behave correctly for previously-unseen inputs.
However, it has been observed that DNNs can react in unexpected and incorrect
ways to even slight perturbations of their
inputs~\cite{SzZaSuBrErGoFe13}. 
This unexpected behavior of DNNs is likely to result in unsafe
systems, or restrict the usage of DNNs in safety-critical applications.
Hence, there is an urgent need for methods that can provide formal 
 guarantees about DNN behavior. 
Unfortunately, manual reasoning about large DNNs is impossible, as 
their structure renders them incomprehensible to humans. Automatic
verification techniques are thus sorely needed, 
but here, the state of the art is a 
severely limiting factor. 

Verifying DNNs is a difficult problem. DNNs are large, non-linear, and
non-convex, and verifying even simple properties about them is
an NP-complete problem (see Section~\ref*{appendix:npc} of the
appendix). DNN verification is experimentally beyond the reach of
general-purpose tools such as \emph{linear programming} (\emph{LP})
solvers or
existing \emph{satisfiability modulo theories} (\emph{SMT})
solvers~\cite{PuTa12,BaIoLaVyNoCr16,HuKwWaWu16}, and thus far, dedicated tools
have only been able to handle very small networks (e.g. a single hidden layer
with only 10 to 20 hidden nodes~\cite{PuTa10, PuTa12}).
 
The difficulty in proving properties about DNNs is caused by the
presence of \emph{activation functions}. A DNN is comprised
of a set of layers of nodes, and the value of each node is
determined by computing a linear combination of values from nodes in the
preceding layer and then applying an activation function to the
result.
These activation functions are non-linear and render the
problem non-convex.
We focus here on DNNs with a specific kind of activation
function, called a \emph{Rectified Linear Unit} (\emph{ReLU})~\cite{NaHi10}. When
the ReLU function is applied to a node with a  
positive value, it returns the value unchanged (the \emph{active} case), but when the value is negative,
the ReLU function returns $0$ (the \emph{inactive} case).
 ReLUs are very widely used~\cite{KrSuHi12,MaHaNg13}, and it has been suggested that  their
 piecewise linearity allows DNNs to 
generalize well to previously unseen inputs~\cite{FoBeCu16, GlBoBe11,
  NaHi10, JaKaLe09}.
Past efforts at verifying properties of DNNs with ReLUs have
had to make significant simplifying assumptions~\cite{HuKwWaWu16, BaIoLaVyNoCr16} ---
for instance, by considering only small input regions in which all ReLUs are
fixed at either the active or inactive state~\cite{BaIoLaVyNoCr16}, hence making the problem
convex but at the cost of being able to verify only an approximation of the
desired property.

We propose a novel, scalable, and efficient algorithm for
verifying properties of DNNs with ReLUs.  We address the issue of
the activation functions head-on, by extending the simplex algorithm ---
a standard algorithm for solving LP instances --- to support ReLU
constraints. This is achieved by leveraging the piecewise linear
nature of ReLUs and attempting to 
gradually satisfy the constraints that they impose
as the algorithm searches for a feasible solution.
We call the algorithm \emph{Reluplex}, for ``ReLU with Simplex''.

The problem's NP-completeness means
that we must expect the worst-case performance of the algorithm to be poor.
However, as is often the case with SAT and SMT solvers, the performance in
practice can be quite reasonable; in particular, our experiments
show that during the search for a solution, many of the 
ReLUs can be ignored or even discarded altogether,
reducing the search space by an order of
magnitude or more. Occasionally, Reluplex will still need to \emph{split} on a specific
ReLU constraint --- i.e., guess that it is either active or
inactive, and possibly backtrack later if the choice leads to a
contradiction.

We evaluated Reluplex on a family of 45 real-world DNNs,
developed as an early prototype for the next-generation airborne collision avoidance system
 for unmanned aircraft
ACAS Xu~\cite{JuLoBrOwKo16}. These fully connected DNNs have 8 layers and 300 ReLU
nodes each, and are intended to be run onboard aircraft. They take in sensor data indicating the speed and
present course of the aircraft (the \emph{ownship}) and that of any
nearby intruder aircraft, and issue appropriate navigation
advisories.
 These advisories indicate whether the
aircraft is clear-of-conflict, in which case the present course can be
maintained, or whether it should turn to
 avoid collision. We successfully proved several properties of
 these networks, e.g. that a clear-of-conflict advisory will always be
 issued if the intruder is sufficiently far away or that it will never be
 issued if the intruder is sufficiently close and on a collision
 course with the ownship. 
Additionally, we were able to prove certain \emph{robustness}
properties~\cite{BaIoLaVyNoCr16} of the networks, meaning that small
adversarial perturbations do not change the advisories produced for
certain inputs.

Our contributions can be summarized as follows. We
\begin{inparaenum}[(i)]
\item present Reluplex, an SMT solver for a theory of linear
  real arithmetic with ReLU constraints;
\item show how DNNs and properties of interest can be encoded as
  inputs to Reluplex;
\item discuss several implementation details that are crucial to performance and scalability, such as the
  use of floating-point arithmetic, bound derivation for ReLU variables,
  and conflict analysis; and 
\item conduct a thorough evaluation on the DNN implementation of
  the prototype ACAS Xu system, demonstrating the ability of Reluplex to scale
  to DNNs that are an order of magnitude larger than those that can be analyzed using
  existing techniques.
\end{inparaenum}

The rest of the paper is organized as follows. We begin with some background
on DNNs, SMT, and simplex in
Section~\ref{sec:background}. 
 The abstract Reluplex algorithm is described in 
Section~\ref{sec:reluplex}, with key implementation details
highlighted in Section~\ref{sec:implementation}.
We then 
describe
the ACAS Xu system and its prototype DNN implementation that we used as a case-study
in Section~\ref{sec:acasxu},
followed by
experimental results in Section~\ref{sec:evaluation}.
 Related work is discussed in Section~\ref{sec:relatedWork}, and we
conclude in Section~\ref{sec:conclusion}.

\section{Background}
\label{sec:background}

\subsubsection{Neural Networks.}
Deep neural networks (DNNs) are comprised of an input layer, an output layer, 
and multiple hidden layers in between. A layer is comprised of multiple
nodes, each connected to 
nodes from the preceding layer using a predetermined set of weights
(see Fig.~\ref{fig:fullyConnectedNetwork}).
Weight selection is crucial, and is performed during a
\emph{training} phase (see, e.g.,~\cite{FoBeCu16} for an overview).
By assigning values to inputs
and then feeding them forward through the network, values for each layer
can be computed from the values of the previous layer, finally resulting in values
for the outputs.

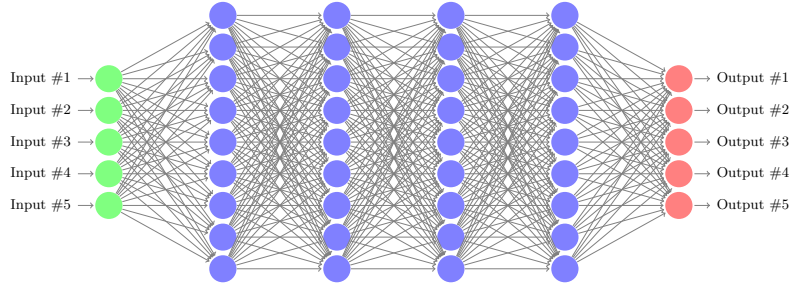
\begin{figure}
\begin{center}
\scalebox{0.6}{
\def\layersep{2.5cm}
\def\vertSepFactory{0.7}
\begin{tikzpicture}[shorten >=1pt,->,draw=black!50, node distance=\layersep]
    \foreach \name / \y in {1,...,5}
        \node[input neuron, pin=left:Input \#\y] (I-\name) at (0,-\vertSepFactory * \y) {};

    \foreach \name / \y in {1,...,9}
        \path[yshift=1.4cm]
            node[hidden neuron] (H1-\name) at (1*\layersep,-\vertSepFactory * \y cm) {};

    \foreach \name / \y in {1,...,9}
        \path[yshift=1.4cm]
            node[hidden neuron] (H2-\name) at (2*\layersep,-\vertSepFactory * \y cm) {};

    \foreach \name / \y in {1,...,9}
        \path[yshift=1.4cm]
            node[hidden neuron] (H3-\name) at (3*\layersep,-\vertSepFactory * \y cm) {};

    \foreach \name / \y in {1,...,9}
        \path[yshift=1.4cm]
            node[hidden neuron] (H4-\name) at (4*\layersep,-\vertSepFactory * \y cm) {};

    \foreach \name / \y in {1,...,5}
        \node[output neuron,pin={[pin edge={->}]right:Output \#\y}]
        (O-\name) at (5*\layersep, -\vertSepFactory * \y cm) {};

    \foreach \source in {1,...,5}
        \foreach \dest in {1,...,9}
            \path (I-\source) edge (H1-\dest);

    \foreach \source in {1,...,9}
        \foreach \dest in {1,...,9}
            \path (H1-\source) edge (H2-\dest);

    \foreach \source in {1,...,9}
        \foreach \dest in {1,...,9}
            \path (H2-\source) edge (H3-\dest);

    \foreach \source in {1,...,9}
        \foreach \dest in {1,...,9}
            \path (H3-\source) edge (H4-\dest);

    \foreach \source in {1,...,9}
        \foreach \dest in {1,...,5}
            \path (H4-\source) edge (O-\dest);

\end{tikzpicture}
}
\caption{A fully connected DNN with 5 input nodes (in green), 5 output
  nodes (in red), and
  4 hidden layers containing a total of 36 hidden nodes (in blue).
}
\label{fig:fullyConnectedNetwork}
\end{center}
\end{figure}
 
The value of each hidden node in the network is determined by
calculating a linear combination of node values from the previous
layer, and then applying a non-linear \emph{activation function}~\cite{FoBeCu16}.
Here, we focus on the 
Rectified Linear Unit (ReLU) activation function~\cite{NaHi10}.
 When a ReLU activation function is applied to a node,
that node's value is calculated as
the maximum of the linear combination of nodes from the previous
layer and $0$. We can thus regard ReLUs as the 
function $\relu{}(x) = \max{}(0, x)$.

Formally, for a DNN $N$, we
use $n$ to denote the number of layers and $s_i$ to denote the size of
layer $i$ (i.e., the number of its nodes).
 Layer $1$ is the input layer, layer $n$ is the output layer, and
 layers $2,\ldots,n-1$ are the hidden layers. 
The value of the $j$-th node of layer $i$ is denoted $v_{i,j}$ 
and the column vector $[v_{i,1},\ldots,v_{i,s_i}]^T$ is denoted
$V_i$.
Evaluating $N$ entails calculating $V_n$ for a given assignment $V_1$
of the input layer. This
is performed by propagating the input values through the network
using predefined weights and biases, and applying the activation
functions --- ReLUs, in our case. Each layer $2\leq i\leq n$ has a weight
matrix $W_i$ of size $s_{i}\times s_{i-1}$ and a bias vector $B_i$ of size
$ s_i$, and its values are given by
$
V_i = \relu{}(W_i  V_{i-1} + B_i),
$
with the ReLU
function being applied element-wise. 
This rule is applied repeatedly for each layer until $V_n$ is
calculated.
When the weight matrices $W_1,\ldots
W_n$ do not have any zero entries, the network is said to be \emph{fully
  connected} (see Fig.~\ref{fig:fullyConnectedNetwork} for an illustration).

Fig.~\ref{fig:runningExample} depicts a small network that we will use
as a running example.
The network has one input node, one
output node and a single hidden layer with two nodes. 
The bias vectors are set
to $0$ and are ignored, and the weights are shown for each edge.
The ReLU function is
applied to each of the hidden nodes.
It is possible to show that, due to the effect of the ReLUs,
the network's output is always identical to its input:
$v_{31}\equiv v_{11}$.

\begin{figure}[htp]
  \begin{center}
    \scalebox{1} {
      \def\layersep{1.5cm}
    \begin{tikzpicture}[shorten >=1pt,->,draw=black!50, node distance=\layersep,font=\footnotesize]

      \node[input neuron] (I-1) at (0,-1) {$v_{11}$};

      \path[yshift=0.5cm] node[hidden neuron] (H-1)
      at (\layersep,-1 cm) {$v_{21}$};
      \path[yshift=0.5cm] node[hidden neuron] (H-2)
      at (\layersep,-2 cm) {$v_{22}$};

      \node[output neuron] at (2*\layersep, -1) (O-1) {$v_{31}$};

      \path (I-1) edge[] node[above] {$1.0$} (H-1);
      \path (I-1) edge[] node[below] {$-1.0$} (H-2);
      \path (H-1) edge[] node[above] {$1.0$} (O-1);
      \path (H-2) edge[] node[below] {$1.0$} (O-1);

      \node[annot,above of=H-1, node distance=1cm] (hl) {Hidden layer};
      \node[annot,left of=hl] {Input layer};
      \node[annot,right of=hl] {Output layer};
    \end{tikzpicture}
    }
    \captionsetup{size=small}
    \captionof{figure}{A small neural network.}
    \label{fig:runningExample}
  \end{center}
\end{figure}
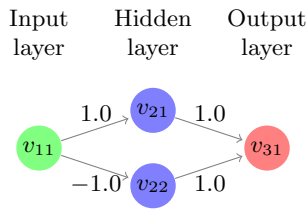

\subsubsection{Satisfiability Modulo Theories.}
We present our algorithm as a theory solver in the context of satisfiability
modulo theories (SMT).\footnote{Consistent with most treatments of SMT, we assume many-sorted first-order logic with
equality as our underlying formalism (see, e.g.,~\cite{BSST09} for details).}
A \emph{theory} is a pair $T = (\Sigma, \Mo)$ where 
$\Sigma$ is a signature and $\Mo$ is a class of $\Sigma$-interpretations,
the \emph{models} of $T$,
that is closed under variable reassignment.
A $\Sigma$-formula $\varphi$ is 
$T$-\emph{satisfiable} (resp., $T$-\emph{unsatisfiable})
if it is satisfied by some (resp., no) interpretation in $\Mo$.  In this paper,
we consider only \emph{quantifier-free} formulas.
The SMT problem is the problem of determining the $T$-satisfiability of a formula for a given
theory $T$.

Given a theory $T$ with signature $\Sigma$, the DPLL($T$)
architecture~\cite{NiOlTi06DPLLT} provides a generic approach for
determining the $T$-satisfiability of $\Sigma$-formulas.
In DPLL($T$), a Boolean satisfiability (SAT) engine operates on a Boolean
abstraction of the formula, performing Boolean propagation,
case-splitting, and Boolean conflict resolution. The SAT engine is
coupled with 
a dedicated \emph{theory solver}, which
checks the $T$-satisfiability of the decisions made by the SAT engine.
 \emph{Splitting-on-demand}~\cite{BaNiOlTi06}
extends DPLL($T$) by allowing theory solvers to delegate
case-splitting to the SAT engine in a generic and modular way.  In
Section~\ref{sec:reluplex}, we present 
our algorithm as a deductive calculus (with splitting rules) operating on conjunctions of literals.
The DPLL($T$) and splitting-on-demand mechanisms can then be
used to obtain a full decision procedure for arbitrary formulas.

\subsubsection{Linear Real Arithmetic and Simplex.}
In the context of DNNs, a particularly relevant theory is that
of real arithmetic, which we denote as $\tr$. $\tr{}$ consists of the
signature containing all rational number constants and the symbols 
 $\{+,-,\cdot,\leq,\geq\}$, paired with the standard model of the real numbers.
We focus on \emph{linear} formulas: formulas over $\tr{}$ with the additional
restriction that the multiplication symbol $\cdot$ can only appear if at least
one of its operands is a rational constant. Linear atoms
can always be rewritten into the form $\sum_{x_i\in
  \allvars}c_ix_i\bowtie d$, for $\bowtie\ \in\{=,\leq,\geq\}$, where
$\allvars$ is a set of variables and $c_i,d$ are rational constants.

The simplex method~\cite{Dantzig1963} is a standard and highly efficient
decision procedure for determining the $\tr{}$-satisfiability of conjunctions of linear
atoms.\footnote{There exist SMT-friendly extensions of simplex (see e.g.~\cite{Ki14}) which can handle
$\tr{}$-satisfiability of arbitrary literals, including strict inequalities and
disequalities, but we omit these extensions here for simplicity (and without loss of generality).}
Our
algorithm extends simplex, and so we begin with an abstract
calculus for the original algorithm (for a more thorough description see,
e.g.,~\cite{Va96}).
The rules of the calculus operate over data structures we call
\emph{configurations}.  For a given set of variables $\allvars = \{x_1,\ldots,x_n\}$,
a simplex configuration is either one of the distinguished symbols
$\{\sat{},\unsat{}\}$ or a tuple $\langle \basic, T, \lb,
\ub, \assignment\rangle$, where:
$\basic\subseteq \allvars$ is a set of basic variables;
$T$, the \emph{tableau}, contains for each $x_i\in\basic$
an equation $x_i =
  \sum_{x_j\notin\basic} c_j x_j$;
  $\lb, \ub$ are mappings that assign each
  variable $x\in\allvars$ a lower and an upper bound, respectively; and
  $\assignment$, the \emph{assignment}, maps each variable $x\in\allvars$
  to a real value. The initial configuration (and in particular the initial tableau
  $T_0$) is derived from a conjunction of input atoms as follows: for each atom $\sum_{x_i\in
  \allvars}c_ix_i\bowtie d$, a new basic variable $b$ is introduced, the
  equation $b=\sum_{x_i\in \allvars}c_ix_i$ is added to the tableau, and $d$ is
  added as a bound for $b$ (either upper, lower, or both, depending on
  $\bowtie$).  
  The initial assignment is set to $0$ for all variables, ensuring that all
  tableau equations hold (though variable bounds may be violated).

The tableau $T$ can be regarded as a matrix expressing
each of the basic variables (variables in $\basic{}$) as a linear combination of non-basic
variables (variables in $\allvars{}\setminus\basic{}$). The rows of
$T$ correspond to the variables in $\basic{}$ and its columns to
those of $\allvars{}\setminus\basic{}$. For $x_i\in\basic{}$ and
$x_j\notin\basic{}$ we denote by $T_{i,j}$ the coefficient
$c_j$ of $x_j$ in the equation $x_i=\sum_{x_j\notin\basic} c_j x_j$.
The tableau is changed via pivoting: the switching of a basic variable
$x_i$ (the \emph{leaving} variable) with a non-basic variable $x_j$
(the \emph{entering} variable) for which $T_{i,j}\neq 0$. A
\pivotOperation{}($T,i,j$) operation returns a new tableau in which the equation $x_i=\sum_{x_k\notin\basic}
c_k x_k$ has been replaced by the equation
$
x_j = \frac{x_i}{c_j} - \sum_{x_k\notin\basic, k\neq j}\frac{c_k}{c_j}x_k
$, and in which every occurrence of $x_j$ in each of the other equations has been
replaced by the right-hand side of the new equation (the resulting expressions
are also normalized to retain the tableau form).
The variable assignment $\assignment{}$ is changed via \emph{update} operations that are
applied to non-basic variables:
for $x_j\notin\basic{}$, an \updateOperation{}($\assignment, x_j,\delta$) operation
returns
an updated
assignment $\assignment'$ identical to $\assignment$, except that
$\assignment'(x_j)=\assignment(x_j)+\delta$ and for
every $x_i\in \basic$, we have
$
\assignment'(x_i)=\assignment(x_i)+\delta\cdot T_{i,j}.
$
To simplify later presentation we also denote:
\begin{align*}
\slack^+(x_i) 
&=
\{
x_j\notin\basic{}\ |\ 
(T_{i,j}>0
\wedge
\assignment(x_j) < \ub(x_j))
\vee
(T_{i,j}<0
\wedge
\assignment(x_j) > \lb(x_j))
\\
\slack^-(x_i) 
&=
\{
x_j\notin\basic{}\ |\ 
(T_{i,j}<0
\wedge
\assignment(x_j) < \ub(x_j))
\vee
(T_{i,j}>0
\wedge
\assignment(x_j) > \lb(x_j))
\end{align*}

The rules of the simplex calculus are provided in
Fig.~\ref{fig:abstractSimplex} in \emph{guarded assignment form}.  A rule
applies to a configuration $S$ if all of the rule's premises hold for $S$.  A
rule's conclusion describes how each component of $S$ is changed, if at all.
When $S'$ is the result of applying a rule to $S$, we say that $S$ derives $S'$.  A sequence of configurations $S_i$
where each $S_i$ derives $S_{i+1}$ is called a \emph{derivation}.

\begin{figure}[t]
\begin{centering}
\scriptsize

\pivot{1}
\drule{
x_i\in\basic,
\ \ 
\assignment (x_i) < \lb(x_i), 
\ \ 
x_j\in\slackPlus(x_i)
}
{
T := \pivotOperation{}(T,i,j), 
\ \ 
\basic := \basic\cup \{x_j\} \setminus \{x_i\}
}
\medskip

\pivot{2}
\drule{
x_i\in\basic,
\ \ 
\assignment (x_i) > \ub(x_i),
\ \
x_j\in\slackMinus(x_i)
}
{
T:=\pivotOperation{}(T,i,j), 
\ \ 
\basic := \basic\cup \{x_j\} \setminus \{x_i\}
}
\medskip 

\update
\drule{
x_j\notin\basic, 
\ \ 
\assignment(x_j) < \lb(x_j) \vee \assignment(x_j) > \ub(x_j), 
\ \ 
\lb(x_j) \leq \assignment(x_j) + \delta \leq \ub(x_j)
}
{
\assignment := \updateOperation(\assignment, x_j, \delta)
}
\medskip

\failure
\drule{
x_i\in\basic,
\ \ 
( \assignment (x_i) < \lb(x_i)
\ \wedge \
\slackPlus(x_i)=\emptyset )
\vee
( \assignment (x_i) > \ub(x_i)
\ \wedge \
\slackMinus(x_i)=\emptyset )
}
{
\unsat{}
}
\medskip

\success
\drule{
\forall x_i\in\allvars. \ 
\lb(x_i) \leq \assignment(x_i) \leq \ub(x_i)
}
{
\sat{}
}
\caption{Derivation rules for the abstract simplex algorithm.}
\label{fig:abstractSimplex}
\end{centering}
\end{figure}

The \update{} rule (with appropriate values of $\delta$) is used to
enforce that non-basic 
variables satisfy their bounds. Basic variables
cannot be directly updated.  Instead,
if a basic variable $x_i$ is too small or too great,
either the \pivot{1} or the \pivot{2} rule is applied,
respectively, to pivot it with a non-basic variable $x_j$.
This makes $x_i$ non-basic so that its assignment can be adjusted
using the \update{} rule.  Pivoting is only allowed when $x_j$ affords
\emph{slack}, that is, the assignment for $x_j$ can be adjusted to bring $x_i$ closer to
its bound without violating its own bound.
Of course, once pivoting occurs
and the \update rule is used to bring $x_i$ within its bounds, other variables
(such as the now basic $x_j$) may be sent outside their bounds,
in which case they must be corrected in a later iteration.
If a basic variable is out of bounds, but none of the non-basic variables
 affords it any slack,
then the \failure rule applies and the problem is unsatisfiable.
Because the tableau is only changed by scaling and adding rows, 
the set of variable assignments that satisfy its equations is always
kept identical to that of $T_0$.
Also, the \updateOperation{} operation guarantees that $\assignment{}$
continues to satisfy the equations of $T$.
Thus, if all variables are within bounds then 
the \success rule can be applied, indicating that $\alpha$ constitutes a satisfying assignment for the original problem.

It is well-known that the simplex calculus 
 is \emph{sound}~\cite{Va96} (i.e. if a derivation ends in $\sat{}$ or $\unsat{}$, then the original problem is satisfiable or unsatisfiable,
respectively) and \emph{complete} (there always exists a derivation
ending in either $\sat{}$
or $\unsat{}$ from any starting configuration).
Termination can be guaranteed if certain strategies are
used in applying the transition rules --- in particular in
picking the leaving and entering variables when 
 multiple options exist~\cite{Va96}. Variable selection strategies are also known to have a
dramatic effect on performance~\cite{Va96}. We note that the version of
simplex described above is usually referred to as \emph{phase one}
simplex, and is usually followed by a \emph{phase two} in which the
solution is optimized according to a cost
function. However, as we are only considering satisfiability,
phase two is not required.

\section{From Simplex to Reluplex}
\label{sec:reluplex}

The simplex
algorithm described in Section~\ref{sec:background} is an efficient
means for solving 
problems that can be encoded as a conjunction of atoms.
Unfortunately, while the weights, biases, and certain properties
of DNNs can be encoded this way, the non-linear ReLU functions cannot. 

When a theory solver operates within an SMT solver, input
atoms can be embedded in arbitrary Boolean structure.
 A na\"ive approach is then to encode ReLUs using disjunctions, which is
 possible because ReLUs are piecewise linear.
However, this encoding requires the
SAT engine within the SMT solver to enumerate the different cases. In the worst
case, for a DNN with $n$ ReLU nodes, the solver ends up splitting the problem
into $2^n$ sub-problems, each of which is a conjunction of atoms.
As observed by us and
others~\cite{BaIoLaVyNoCr16,HuKwWaWu16}, this theoretical worst-case behavior is
also seen in practice, and hence this approach is practical only for
very small networks.  A
similar phenomenon occurs when encoding DNNs as mixed integer
problems (see Section~\ref{sec:evaluation}).

We take a different route and extend the theory $\tr{}$ to a theory $\trr{}$ of reals and
ReLUs.  $\trr{}$ is almost identical to $\tr{}$, except that its signature
additionally includes the binary predicate $\relu{}$ with the interpretation:
$\relu{}(x,y)$ iff $y=\max{}(0,x)$.  Formulas are then assumed to contain atoms that
are either linear inequalities or applications of the $\relu{}$ predicate to
linear terms.

DNNs and their (linear) properties can be
directly encoded as conjunctions of $\trr{}$-atoms. 
The main idea is to encode a single ReLU node $v$ as a \emph{pair} of
variables, $v^b$ and $v^f$, and then assert $\relu{}(v^b,v^f)$.
$v^b$, the \emph{backward-facing} variable, is used to express the
connection of $v$ to nodes from
the preceding layer; whereas
 $v^f$, the
\emph{forward-facing} variable, 
 is used for the connections of $x$ to 
the following layer (see Fig.~\ref{fig:nnRelus}).
The rest of this section is devoted
to presenting an efficient algorithm, Reluplex, for deciding the
satisfiability of a conjunction of such atoms.

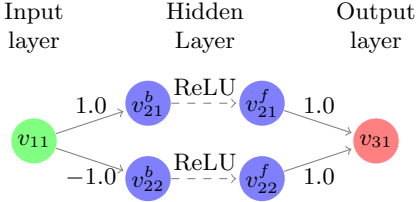
\begin{figure}[htp]
  \begin{center}
    \scalebox{1} {
      \def\layersep{1.5cm}
        \begin{tikzpicture}[shorten >=1pt,->,draw=black!50, node distance=\layersep,font=\footnotesize]
      \node[input neuron] (I-1) at (0,-1) {$v_{11}$};

      \path[yshift=0.5cm] node[hidden neuron] (H-1b)
      at (\layersep,-1 cm) {$v^b_{21}$};
      \path[yshift=0.5cm] node[hidden neuron] (H-2b)
      at (\layersep,-2 cm) {$v^b_{22}$};

      \path[yshift=0.5cm] node[hidden neuron] (H-1f)
      at (2*\layersep,-1 cm) {$v^f_{21}$};
      \path[yshift=0.5cm] node[hidden neuron] (H-2f)
      at (2*\layersep,-2 cm) {$v^f_{22}$};

      \node[output neuron] at (3*\layersep, -1) (O-1) {$v_{31}$};

      \path (I-1) edge[] node[above] {$1.0$} (H-1b);
      \path (I-1) edge[] node[below] {$-1.0$} (H-2b);
      \path (H-1f) edge[] node[above] {$1.0$} (O-1);
      \path (H-2f) edge[] node[below] {$1.0$} (O-1);
      \draw[dashed] (H-1b) -- node[above] {ReLU} (H-1f);
      \draw[dashed] (H-2b) -- node[above] {ReLU} (H-2f);

      \node[annot,above of=I-1, node distance = 1.5cm] (hl) {Input layer};
      \node[annot,above of=O-1, node distance = 1.5cm] (hr) {Output
        layer};

      \node[annot] at ($(hl) !.5! (hr)$) (hm) {Hidden Layer};
        \end{tikzpicture}
      }
\caption{The network from Fig.~\ref{fig:runningExample}, with ReLU
  nodes split into 
  backward- and forward-facing variables.}
\label{fig:nnRelus}
\end{center}
\end{figure}

\subsubsection{The Reluplex Procedure.}
As with simplex, Reluplex
 allows variables to temporarily violate their bounds
as it iteratively looks for a
feasible variable assignment. However, Reluplex also allows 
 variables that are members of ReLU pairs to temporarily
violate the ReLU semantics.
Then, as it iterates, Reluplex repeatedly picks variables that are either
 out of bounds or that violate a ReLU, and corrects them using
\pivot{} and \update{} operations.

For a given set of variables $\allvars = \{x_1,\ldots,x_n\}$,
a Reluplex configuration is either one of the distinguished symbols
$\{\sat{},\unsat{}\}$ or a tuple
$\langle \basic, T, \lb, \ub, \assignment, \reluSet \rangle$,
where $\basic, T,\lb,\ub$ and $\assignment$ are as before, and
$\reluSet\subset \allvars\times \allvars$ is the set of ReLU
connections.  The initial configuration for a conjunction of atoms
is also obtained as before except that $\langle x,y\rangle\in \reluSet$ iff
$\relu{}(x,y)$ is an atom.
The simplex transition rules 
\pivot{1}, \pivot{2} and \update{} are included also in Reluplex, as
they are designed to handle out-of-bounds violations. We replace the
\success{} rule with the \reluSuccess{} rule and add rules for
handling ReLU violations, as depicted in
Fig.~\ref{fig:abstractReluplex}. The \updateb{} and \updatef{} rules
allow a broken ReLU connection to be corrected by updating the backward- or
forward-facing variables, respectively, provided that these variables are
non-basic. The \pivotForRelu{} rule allows a basic variable appearing in a ReLU to be
pivoted so that
either \updateb{} or \updatef{} can be applied (this is needed to make progress when both variables in a ReLU are
basic and their assignments do not satisfy the ReLU semantics).
The \reluSplit{} rule is used for splitting
on certain ReLU connections, guessing whether they are
active (by setting $\lb(x_i):=0$)  or inactive (by setting
$\ub(x_i):=0$).

\begin{figure}[t]
\begin{centering}
\scriptsize

\updateb
\drule{
x_i\notin\basic, 
\ \
\langle x_i, x_j\rangle \in \reluSet,
\ \ 
\assignment(x_j)\neq \max{}(0,\assignment(x_i)),
\ \ 
 \assignment(x_j)\geq 0
}
{
\assignment := \updateOperation(\assignment, x_i, \assignment(x_j)-\assignment(x_i))
}
\medskip

\updatef
\drule{
x_j\notin\basic, 
\ \
\langle x_i, x_j\rangle \in \reluSet,
\ \ 
\assignment(x_j)\neq \max{}(0,\assignment(x_i))
}
{
\assignment := \updateOperation(\assignment, x_j, \max{}(0, \assignment(x_i)) -\assignment(x_j))
}
\medskip

\pivotForRelu
\drule{
x_i\in\basic,
\ \ 
\exists x_l. \ \langle x_i, x_l\rangle \in \reluSet \vee \langle x_l, x_i\rangle \in \reluSet,
\ \ 
x_j\notin\basic,
\ \ 
T_{i,j} \neq 0
}
{
T:=\pivotOperation{}(T,i,j), 
\ \ 
\basic := \basic\cup \{x_j\} \setminus \{x_i\}
}
\medskip

\reluSplit
\drule{
\langle x_i, x_j\rangle \in \reluSet, 
\ \ 
\lb(x_i)<0,
\ \ 
 \ub(x_i)>0
}
{
\ub(x_i):= 0 
\qquad
\lb(x_i):= 0 
}
\medskip

\reluSuccess
\drule{
\forall x\in\allvars. \ 
\lb(x) \leq \assignment(x) \leq \ub(x), 
\ \ 
\forall \langle x^b,x^f \rangle \in \reluSet. \
\assignment(x^f) = \max{}(0, \assignment(x^b))
}
{
\sat{}
}
\caption{Additional derivation rules for the abstract Reluplex algorithm.}
\label{fig:abstractReluplex}
\end{centering}
\end{figure}

Introducing splitting  means that derivations are
no longer linear.  Using the notion of derivation trees, we
can show that Reluplex is sound and complete (see 
Section~\ref*{appendix:soundness} of the
appendix).
In practice, splitting can be managed by a SAT engine with splitting-on-demand~\cite{BaNiOlTi06}.
The na\"ive approach mentioned at the beginning of this section can be
simulated by applying the \reluSplit{} rule eagerly until it no longer applies
and then solving each derived sub-problem separately (this reduction trivially
guarantees termination just as do branch-and-cut techniques in mixed integer solvers~\cite{PaRi91}).
However, a more scalable
 strategy is to try to fix broken ReLU pairs
using the \updateb{} and \updatef{} rules first, and split
only when the number of updates to a specific ReLU pair
exceeds some threshold. Intuitively, this is likely
to limit splits to ``problematic'' ReLU pairs,
while still guaranteeing termination (see 
Section~\ref*{appendix:termination} of the
appendix).
Additional details appear in Section~\ref{sec:evaluation}.

\subsubsection{Example.}
\sloppy
To illustrate the use of the derivation rules, we use Reluplex to solve a
simple example. Consider the network in
Fig.~\ref{fig:nnRelus}, and suppose we wish to check whether 
it is possible to satisfy $v_{11}\in [0,1]$ and $v_{31}\in [0.5,1]$. 
As we know that the network outputs its input unchanged ($v_{31}\equiv v_{11}$), we expect
Reluplex to be able to derive \sat{}.
The initial Reluplex configuration is obtained
by introducing new basic variables $a_1,a_2,a_3$,
and encoding the network with the equations:
\[
  a_1 = -v_{11} + v^b_{21}  \qquad
  a_2 = v_{11} + v^b_{22}  \qquad
  a_3 = - v^f_{21} - v^f_{22} + v_{31}
\]
The equations above form the initial tableau $T_0$, and the
initial set of basic variables is
$\basic{} = \{a_1,a_2,a_3\}$. 
The set of ReLU connections is 
$\reluSet=\{
\langle v^b_{21},v^f_{21} \rangle,
\langle v^b_{22},v^f_{22} \rangle
\}$.
The initial assignment of all variables is set to $0$.
The lower and upper bounds of the
basic variables are set to 0, in order to enforce the equalities
that they represent. The bounds for the input and output variables are set
according to the problem at hand; and the hidden variables are
unbounded, except that forward-facing variables are, by
definition, non-negative:

\noindent
\begin{center}
\scalebox{1}{
\begin{tabular}{l|ccccccccc}
  variable &
  $v_{11}$ &
  $v_{21}^b$ &
  $v_{21}^f$ &
  $v_{22}^b$ &
  $v_{22}^f$ &
  $v_{31}$ &
  $a_1$ &
  $a_2$ &
  $a_3$ 
\\
  \cline{2-10}
  lower bound &
  $0$ &
  $-\infty$ &
  $0$ &
  $-\infty$ &
  $0$ &
  $0.5$ &
  $0$ &
  $0$ &
  $0$ 
\\
  \cline{2-10}
  assignment &
  $0$ &
  $0$ &
  $0$ &
  $0$ &
  $0$ &
  $0$ &
  $0$ &
  $0$ &
  $0$              
\\
  \cline{2-10}
  upper bound &
  $1$ &
  $\infty$ &
  $\infty$ &
  $\infty$ &
  $\infty$ &
  $1$ &
  $0$ &
  $0$ &
  $0$ 
\end{tabular}
}
\end{center}

Starting from this initial configuration, our search strategy is to first fix
any out-of-bounds variables. Variable $v_{31}$ is non-basic and is out of bounds, so we
perform an \update{} step and set it to $0.5$. As a result, $a_3$,
which depends on $v_{31}$, is also set to $0.5$. $a_3$ is now basic
and out of bounds, so we pivot it with $v_{21}^f$, and then update $a_3$ back to
$0$. The tableau now consists of the equations:
\[
  a_1 = -v_{11} + v^b_{21}  \qquad
  a_2 = v_{11} + v^b_{22}  \qquad
  v_{21}^f = -v_{22}^f + v_{31} -a_3  
\]
And the assignment is
$\assignment(v_{21}^f) = 0.5$,
$\assignment(v_{31}) = 0.5$,
and $\assignment(v) = 0$ 
for all other variables $v$. At this
point, all variables are within
their bounds, but the \reluSuccess{} rule does not
apply because
 $\assignment(v_{21}^f) = 0.5 \neq 0 =
\max{}(0,\assignment(v_{21}^b))$. 

The next step is to fix the broken ReLU pair $\langle{v_{21}^b, v_{21}^f}\rangle$. Since $v_{21}^b$ is non-basic, we use \updateb{} to 
increase its value by $0.5$. The assignment becomes
$\assignment(v_{21}^b) = 0.5$,
$\assignment(v_{21}^f) = 0.5$,
$\assignment(v_{31}) = 0.5$,
$\assignment(a_1) = 0.5$,
and $\assignment(v) = 0$ for all other variables $v$. All ReLU constraints hold, but
 $a_1$ is now out of bounds. This is fixed
by pivoting $a_1$ with $v_{11}$ and then updating it. The
resulting tableau is:
\[
  v_{11} = v_{21}^b - a_1 \qquad
  a_2 = v_{21}^b + v^b_{22} - a_1  \qquad
  v_{21}^f = -v_{22}^f + v_{31} -a_3  
\]
Observe that because $v_{11}$ is now basic, it was eliminated from the
equation for $a_2$ and replaced with $v_{21}^b - a_1$.
 The non-zero assignments are now
$\assignment(v_{11}) = 0.5$,
$\assignment(v_{21}^b) = 0.5$,
$\assignment(v_{21}^f) = 0.5$,
$\assignment(v_{31}) = 0.5$,
$\assignment(a_2) = 0.5$. 
Variable $a_2$ is now too large, and so we have a final round of
pivot-and-update: $a_2$ is pivoted with $v_{22}^b$ and then updated back
to $0$. The final tableau and assignments are:

\noindent
\begin{minipage}{.35\textwidth}
\begin{align*}
&  v_{11} = v_{21}^b - a_1 \\
&  v_{22}^b = -v_{21}^b +  a_1  + a_2 \\
&  v_{21}^f = -v_{22}^f + v_{31} -a_3  \\
\end{align*}
\end{minipage}
\begin{minipage}{.50\textwidth}
\begin{center}
\begin{tabular}{l|ccccccccc}
  variable &
  $v_{11}$ &
  $v_{21}^b$ &
  $v_{21}^f$ &
  $v_{22}^b$ &
  $v_{22}^f$ &
  $v_{31}$ &
  $a_1$ &
  $a_2$ &
  $a_3$ 
\\
  \cline{2-10}
  lower bound &
  $0$ &
  $-\infty$ &
  $0$ &
  $-\infty$ &
  $0$ &
  $0.5$ &
  $0$ &
  $0$ &
  $0$ 
\\
  \cline{2-10}
  assignment &
  $0.5$ &
  $0.5$ &
  $0.5$ &
  $-0.5$ &
  $0$ &
  $0.5$ &
  $0$ &
  $0$ &
  $0$              
\\
  \cline{2-10}
  upper bound &
  $1$ &
  $\infty$ &
  $\infty$ &
  $\infty$ &
  $\infty$ &
  $1$ &
  $0$ &
  $0$ &
  $0$ 
\end{tabular}
\end{center}
\end{minipage}
 
\noindent
and the algorithm halts with the feasible solution it has found.
A key observation is that we did not
ever split on any of the ReLU connections.
Instead, it was sufficient to simply use updates to adjust the ReLU variables as needed.

\section{Efficiently Implementing Reluplex}
\label{sec:implementation}
 We next discuss three techniques that
significantly boost the performance of Reluplex: use of tighter bound
derivation,  conflict analysis
 and floating point arithmetic.
 A fourth technique, under-approximation, is discussed in
Section~\ref*{appendix:approximation} of the
appendix.

\subsubsection{Tighter Bound Derivation.}
The simplex and Reluplex procedures naturally lend themselves to
deriving tighter variable bounds as the search progresses~\cite{Ki14}.
Consider a basic variable
$x_i\in\basic{}$ and
 let
$
\posWeights(x_i) =\{x_j\notin\basic{}\ | \ T_{i,j}>0\}
$
 and
 $
\negWeights(x_i) =\{x_j\notin\basic{}\ | \ T_{i,j}<0\}.
$
Throughout the execution, the following rules can be used to derive 
tighter bounds for $x_i$, regardless of the current assignment:

\medskip
\noindent
\scalebox{0.9}{
\learnLB
\drule{
x_i\in \basic{},
\ \
\lb(x_i) <
  \sum_{x_j\in\posWeights(x_i)} T_{i,j} \cdot \lb(x_j)
+
\sum_{x_j\in\negWeights(x_i)} T_{i,j} \cdot \ub(x_j)
}
{
\lb(x_i) :=
 \sum_{x_j\in\posWeights(x_i)} T_{i,j} \cdot \lb(x_j)
+
\sum_{x_j\in\negWeights(x_i)} T_{i,j} \cdot\ub(x_j)
}
}

\medskip
\noindent
\scalebox{0.9}{
\learnUB
\drule{
x_i\in \basic{},
\ \
\ub(x_i) >
 \sum_{x_j\in\posWeights(x_i)} T_{i,j} \cdot\ub(x_j)
+
\sum_{x_j\in\negWeights(x_i)} T_{i,j} \cdot\lb(x_j)
}
{
\ub(x_i) :=
 \sum_{x_j\in\posWeights(x_i)} T_{i,j} \cdot\ub(x_j)
+
\sum_{x_j\in\negWeights(x_i)} T_{i,j} \cdot\lb(x_j)
}
}

\medskip
\noindent
The derived bounds can later be used to derive additional, tighter bounds.

When tighter bounds are derived for ReLU variables, these variables can sometimes
be eliminated, i.e., fixed to the active
or inactive state, without splitting. 
For a ReLU pair $x^f=\relu{}(x^b)$, 
 discovering that either $\lb{}(x^b)$ or $\lb{}(x^f)$
is strictly positive
means that in any
feasible solution this ReLU connection will be active.
Similarly, 
discovering
that $\ub(x^b)<0$ implies inactivity.
 
Bound tightening operations incur overhead, and simplex
implementations often use them sparsely~\cite{Ki14}. In Reluplex,
however, the benefits of eliminating ReLUs justify the cost. 
The actual amount of bound tightening to perform can be determined
heuristically; we describe the heuristic that we used in Section~\ref{sec:evaluation}.

\subsubsection{Derived Bounds and Conflict Analysis.}
Bound derivation can
lead to situations where we learn that
$\lb{}(x)>\ub{}(x)$
for some variable $x$. Such contradictions allow Reluplex
to immediately undo a previous split (or answer \unsat{} if
no previous splits exist).
However, in many cases more than just the previous split can be undone.
For example, if we have performed $8$ nested splits so far, 
it may be that the conflicting bounds for $x$ are the direct result of split
number $5$ but have only just been discovered.
 In this case we can immediately undo splits number 8, 7, and 6. This
 is a particular case of 
\emph{conflict analysis}, which is a 
standard technique in SAT and SMT solvers~\cite{MaSa99}.

\subsubsection{Floating Point Arithmetic.}

SMT solvers typically use precise (as opposed to floating point) arithmetic 
to avoid roundoff errors and guarantee soundness.
Unfortunately, precise computation is usually at least an order of
magnitude slower than its floating point equivalent.
Invoking Reluplex on a large DNN can require millions of 
pivot operations, each of which involves the multiplication
and division of rational numbers, potentially with large numerators or denominators --- making the use of floating point arithmetic
important for scalability.

There are standard techniques for keeping the
roundoff error small when implementing simplex using floating point,
which we incorporated into our implementation.
For example, one important practice is trying to avoid \pivot{} operations
involving the inversion of extremely small numbers~\cite{Va96}. 

 To provide increased confidence that any roundoff error remained
within an acceptable range,
we also added the following safeguards:
  \begin{inparaenum}[(i)]
\item After a certain number of \pivot{} steps we would measure the
   accumulated roundoff error; and
\item If the error exceeded a threshold $M$, we would \emph{restore} the coefficients of
  the current tableau $T$ using the initial tableau $T_0$.
\end{inparaenum}

Cumulative roundoff error can be measured by plugging the
current assignment values for the non-basic variables into the equations 
of the initial tableau $T_0$, 
 using them to calculate the values for every basic variable $x_i$, and then
 measuring by how much these values differ from the 
 current assignment $\assignment(x_i)$. We define the cumulative
 roundoff error as:
\[
 \sum_{x_i\in\basic{}_0} |\assignment(x_i) - \sum_{x_j\notin\basic_0}
 T_{0_{i,j}} \cdot \assignment(x_j)|
\]

 $T$ is restored by starting from 
 $T_0$
and performing a short series of \pivot{} steps that result in the same set of
basic variables as in $T$. 
In general, the shortest sequence of pivot steps to transform $T_0$ to $T$
is much shorter than the series of steps that was followed by
Reluplex --- and hence, although it is also performed using floating
point arithmetic,
it incurs a  
smaller roundoff error. 

The tableau restoration technique serves to increase our confidence in
the algorithm's results when using floating point arithmetic, but it
does not guarantee soundness. Providing true soundness when using
floating point arithmetic remains a future goal (see
Section~\ref{sec:conclusion}).

\section{Case Study: The ACAS Xu System}
\label{sec:acasxu}

Airborne collision avoidance systems are critical for ensuring the
safe operation of aircraft. The \emph{Traffic Alert and Collision Avoidance System}
(\emph{TCAS}) was developed in response to midair collisions between
commercial aircraft, and is currently mandated on all 
large commercial aircraft worldwide~\cite{Kuchar2007}. 
Recent work has focused on creating
a new system, known as \emph{Airborne Collision Avoidance System X}
(\emph{ACAS X})~\cite{DMU2015,Kochenderfer2011atc371}. This system adopts an approach that involves solving
a partially observable Markov decision process to optimize the
alerting logic and further reduce the probability of midair
collisions, while minimizing unnecessary
alerts~\cite{DMU2015,Kochenderfer2011atc371,Kochenderfer2012lljournal}.   

The unmanned variant of ACAS X, known as ACAS Xu, produces horizontal
maneuver advisories. So far, development of ACAS Xu has focused on using a 
large lookup table that maps sensor measurements to advisories~\cite{JuLoBrOwKo16}. However, this 
table requires over 2GB of memory. There is concern about the memory requirements
for certified avionics hardware. To
overcome this challenge, a DNN representation was explored as a potential replacement for the 
table~\cite{JuLoBrOwKo16}. Initial results show a dramatic reduction
in memory requirements without compromising safety. In fact, due to
its continuous nature, the DNN approach can sometimes outperform 
the discrete lookup table~\cite{JuLoBrOwKo16}. 
Recently, in order to reduce lookup time, the DNN approach was
improved further, and the single DNN was replaced by an array of 45 DNNs.
As a result, the original 2GB table can now be substituted with
 efficient DNNs that require less than 3MB of memory. 

A DNN implementation of ACAS Xu presents new certification challenges.
Proving that a set of inputs cannot
produce an erroneous alert is paramount for certifying the system for
use in safety-critical settings.
 Previous certification methodologies
 included exhaustively testing the system in 1.5 million simulated
 encounters~\cite{Kochenderfer2010jgcd}, but this is insufficient for proving that faulty behaviors do not exist
within the continuous DNNs. 
This highlights the need for verifying DNNs and
makes the ACAS Xu DNNs prime candidates on which 
to apply Reluplex.

\subsubsection{Network Functionality.} 
The ACAS Xu system maps input variables to action advisories. Each advisory is assigned a score,
with the lowest score corresponding to the best action.
The input state is composed of seven dimensions (shown
in Fig.~\ref{geo_horizontal}) which represent information 
determined from sensor measurements~\cite{Kochenderfer2011atc371}: 
\begin{inparaenum}[(i)]
	\item $\rho$: Distance from ownship to intruder;
	\item $\theta$: Angle to intruder relative to ownship heading direction;
	\item $\psi$: Heading angle of intruder relative to ownship heading direction;
	\item $v_\text{own}$: Speed of ownship;
	\item $v_\text{int}$: Speed of intruder;
	\item $\tau$: Time until loss of vertical separation; and
	\item $a_\text{prev}$: Previous advisory.
\end{inparaenum}
There are five outputs which represent the different horizontal
advisories that can be given to the ownship: Clear-of-Conflict (COC),
weak right, strong right, weak left, or strong left.
Weak and strong mean heading rates of \SI{1.5}{\degree\per\second} and \SI{3.0}{\degree\per\second}, respectively.

\begin{figure}[htp]
  	\centering
	\scalebox{0.8}{
          \includegraphics{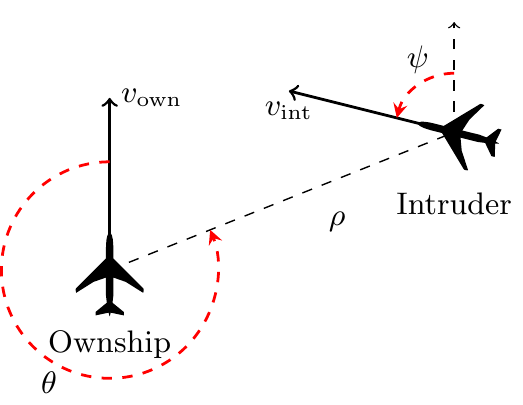}
	}
	\caption{Geometry for ACAS Xu Horizontal Logic Table}
	\label{geo_horizontal}
\end{figure}

The array of 45 DNNs was produced by discretizing $\tau$ and
$a_\text{prev}$, and producing a network for each discretized combination.
Each of these networks thus has five inputs (one for each of the other
dimensions) and five outputs. The DNNs are fully connected, use ReLU activation
functions, and have 6 hidden layers with
a total of 300 ReLU nodes each.

\subsubsection{Network Properties.}
It is desirable to verify that the ACAS Xu networks assign correct scores to
the output advisories in various input domains. 
Fig.~\ref{fig_headon} illustrates this kind of property by showing a
top-down view of a head-on encounter scenario, in which each pixel is colored to represent the best action
if the intruder were at that location. 
We expect the DNN's advisories to be consistent in each of these
regions; however, Fig.~\ref{fig_headon} was generated from a finite
set of input samples, and there may exist other inputs for which a wrong
advisory is produced, possibly leading to collision. Therefore, we
used Reluplex to prove properties from the following categories on the
DNNs:
\begin{inparaenum}[(i)]
	\item The system does not give unnecessary turning advisories;
	\item Alerting regions are uniform and do not contain
          inconsistent alerts; and
	\item Strong alerts do not appear for high $\tau$ values.
\end{inparaenum}

\begin{figure}[htp]
	\centering
	\scalebox{0.8}{
          \includegraphics{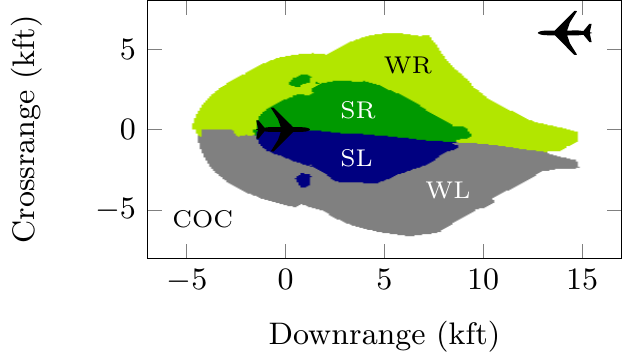}
	}
	\caption{Advisories for a head-on encounter with $a_\text{prev} = \text{COC},\tau=\SI{0}{\second}$.}
	\label{fig_headon}
\end{figure}

\section{Evaluation}
\label{sec:evaluation}
We used a proof-of-concept implementation of Reluplex 
to check realistic properties on the 45 ACAS Xu DNNs.
Our implementation consists of three main logical components:
\begin{inparaenum}[(i)]
\item A simplex engine
  for providing core
  functionality such as tableau representation and pivot and update
  operations;
\item
  A Reluplex engine for driving the search and performing bound derivation, ReLU pivots and
  ReLU updates; and
\item 
  A simple SMT core for providing splitting-on-demand services.
\end{inparaenum}
For the simplex engine we used the GLPK open-source LP
solver\footnote{\url{www.gnu.org/software/glpk/}} with some modifications, 
for instance in order to allow the Reluplex core to perform
bound tightening on tableau equations
calculated by GLPK. Our implementation, together with the experiments
described in this section, is available online~\cite{reluplexCode}.

Our search strategy was to repeatedly fix any out-of-bounds violations first,
and only then correct any violated ReLU constraints (possibly introducing
new out-of-bounds violations). We performed bound
tightening on the entering variable after every pivot
operation,
and performed a more thorough bound tightening
on all the equations in the tableau
 once every
 few thousand pivot steps. 
 Tighter bound derivation proved extremely useful, and we often
observed that after splitting on about 10\% of the
ReLU variables it led to the elimination of all remaining ReLUs.
We counted the number of times a ReLU pair was
fixed via
\updateb{} or \updatef{} or pivoted via \pivotForRelu{}, and
split only when this number reached 5 (a number empirically determined to work
well).
We also implemented conflict analysis and back-jumping.
Finally, we checked the
 accumulated roundoff error
(due to the use of double-precision floating point arithmetic)
 after every $5000$ \pivot{} steps, and
 restored the tableau if the error exceeded $10^{-6}$. Most 
 experiments described below required two tableau restorations or fewer.

We began by comparing our implementation of Reluplex to state-of-the-art solvers: the CVC4,
Z3, Yices and MathSat SMT solvers and the Gurobi LP solver (see Table~\ref{table:smtSolvers}). 
We ran all solvers with a 4 hour timeout on 2 of the ACAS Xu networks
(selected arbitrarily),
trying to solve for 8 
simple satisfiable properties $\varphi_1,\ldots,\varphi_8$, each of
the form $x\geq c$ for a fixed output
variable $x$ and a constant $c$. 
The SMT solvers generally performed poorly, with only Yices and
MathSat successfully solving two instances each. 
We attribute the results to these solvers' lack of direct support for
encoding ReLUs, and to their use of precise arithmetic. Gurobi 
solved 3 instances quickly, but timed out on all the rest. 
Its logs indicated that whenever Gurobi could solve the problem
without case-splitting, it did so quickly; but whenever the problem required
case-splitting, Gurobi would time out. Reluplex was able to solve all
8 instances. 
See Section~\ref*{appendix:encoding} of the
appendix
for the SMT and LP encodings that we used.

\begin{table}
\centering
\caption{
Comparison to SMT and LP solvers. Entries indicate solution time (in seconds).
}
\begin{tabular}[htp]{cc|crlrlrlrlrlrlrlr}
  \toprule
  &&& \multicolumn{1}{c}{$\varphi_1$} 
  && \multicolumn{1}{c}{$\varphi_2$}   
  && \multicolumn{1}{c}{$\varphi_3$}   
  && \multicolumn{1}{c}{$\varphi_4$}   
  && \multicolumn{1}{c}{$\varphi_5$}   
  && \multicolumn{1}{c}{$\varphi_6$}   
  && \multicolumn{1}{c}{$\varphi_7$}   
  && \multicolumn{1}{c}{$\varphi_8$}   
    \\
  \midrule
  CVC4     &&& -    && -    && - && - && -  && - && - && - \\
  Z3       &&& -    && -    && - && - && -  && - && - && - \\
  Yices    &&& 1    && 37   && - && - && -  && - && - && - \\
  MathSat  &&& 2040 && 9780 && - && - && -  && - && - && - \\
  Gurobi   &&& 1    && 1    && 1 && - && -  && - && - && - \\
  Reluplex &&& 8    && 2    && 7 && 7 && 93 && 4 && 7 && 9 \\
  \bottomrule
\end{tabular}%
\label{table:smtSolvers}
\end{table}%

Next, we used Reluplex to test a set of 10 quantitative properties
$\phi_1,\ldots,\phi_{10}$. The properties, described
 below, are formally defined in 
Section~\ref*{appendix:properties} of the
appendix.
Table~\ref{table:acasProperties} depicts for each property the number
of tested networks (specified as part of the property), the test
results and the total duration
(in seconds). The \emph{Stack} and \emph{Splits} columns list the maximal depth of
nested case-splits reached (averaged over the tested networks)
and the total number of case-splits performed, respectively.
For each property, we looked for an input that would violate it;
thus, an \unsat{} result indicates that a property
holds, and a \sat{} result indicates that it does not hold.
In the \sat{} case, the 
satisfying assignment is an example of an input that violates the property.

\begin{table}
\centering
\caption{
Verifying properties of the ACAS Xu networks.
}
\begin{tabular}[htp]{lc|crlclclrlr}
  \toprule
  &&& 
     \multicolumn{1}{c}{Networks}
  && 
     \multicolumn{1}{c}{Result}   
  && 
     \multicolumn{1}{c}{Time}
  && 
     \multicolumn{1}{c}{Stack}
  && 
     \multicolumn{1}{c}{Splits}
  \\
  \midrule
  $\phi_1$     &&& 41 && \unsat{} && 394517 && 47 && 1522384 \\
               &&& 4 && \timeout{} &&  &&  &&  \\
  $\phi_2$     &&& 1 && \unsat{} && 463 && 55 && 88388  \\
               &&& 35 && \sat{} && 82419 && 44 && 284515  \\
  $\phi_3$     &&& 42 && \unsat{}  && 28156 && 22 && 52080 \\
  $\phi_4$     &&& 42 && \unsat{}  && 12475 && 21 && 23940 \\
  $\phi_5$     &&& 1 &&  \unsat{}  && 19355 && 46 && 58914 \\
  $\phi_6$     &&& 1 &&  \unsat{}  && 180288 && 50 && 548496 \\
  $\phi_7$     &&& 1 && \timeout{}   &&  &&  &&  \\
  $\phi_8$     &&& 1 && \sat{} && 40102 && 69 && 116697 \\
  $\phi_9$     &&& 1 && \unsat{} && 99634 && 48 && 227002 \\
  $\phi_{10}$    &&& 1 && \unsat{} && 19944 && 49 && 88520 \\
  \bottomrule
\end{tabular}%
\label{table:acasProperties}
\end{table}%

Property $\phi_1$ states that if the intruder is distant and is
significantly slower than the ownship, the score of a COC advisory will
always be below a certain fixed threshold (recall that the best action
has the lowest score). Property $\phi_2$ states
that under similar conditions, the score for COC can never be maximal,
meaning that it can never be the worst action to take.
This property was discovered not to hold for 35 networks, 
but this was later determined to be acceptable behavior:
the DNNs 
have a strong bias for producing the same advisory they
had previously produced, and this
can result in advisories other than COC even for far-away intruders
if the previous advisory was also something other than COC.
Properties $\phi_3$ and $\phi_4$ deal with situations where the
intruder is directly ahead of the ownship, and state that
the DNNs will never issue a COC advisory.

Properties $\phi_5$ through $\phi_{10}$ each involve a single network, and
check for consistent behavior in a specific input region. 
For example, $\phi_5$ states that if the intruder is near and approaching from the
left, the network advises ``strong right''.
Property $\phi_7$, on which we timed out,
 states that when the vertical separation is large
the network will never advise a strong turn. The large input
domain and the particular network proved difficult to verify. Property
$\phi_8$ states that for a large vertical separation and a previous
``weak left'' advisory, the network will either output COC or
continue advising ``weak left''. Here, we were
able to find a counter-example, exposing an input on which the DNN
was inconsistent with the lookup table. 
This confirmed the existence of a discrepancy that had also been seen in
simulations, and which
 will be addressed by retraining the DNN.
We observe that for all 
properties, the maximal depth of nested splits was always well below the total
number of ReLU nodes, 300, 
illustrating the fact that Reluplex did not split on many of
them. Also, the total number of case-splits indicates that
large portions of the search space were pruned.

Another class of properties that we tested is \emph{adversarial
  robustness} properties. DNNs have been shown
to be susceptible to adversarial inputs~\cite{SzZaSuBrErGoFe13}: 
correctly classified inputs that an adversary
slightly perturbs, leading to their misclassification by the
network. 
 Adversarial robustness is thus a safety
consideration, and adversarial inputs can be used to train the
network further, making it more robust~\cite{GoShSz14}.
There exist approaches for finding adversarial
inputs~\cite{GoShSz14,BaIoLaVyNoCr16}, but the ability to verify their
absence is limited.

We say that a network is $\delta$\emph{-locally-robust} at input
point $\vec{x}$ if for every $\vec{x'}$ such that $\| \vec{x} -
\vec{x'}\|_{\infty}\leq\delta$, the network assigns the same label to $\vec{x}$ 
and $\vec{x'}$.  
In the case of the ACAS Xu DNNs, this means that the same output has the lowest
score for both $\vec{x}$ and $\vec{x'}$.
Reluplex can be used 
to prove local robustness for a given $\vec{x}$ and $\delta$, as
depicted in Table~\ref{table:localRobustness}.
We used one of the ACAS Xu networks, and tested
combinations of 5 arbitrary points and 5 values of
$\delta$. \sat{} results show that 
Reluplex found an adversarial input within the prescribed
neighborhood, and \unsat{} results indicate
that no such inputs exist. 
Using binary search on values of $\delta$,
Reluplex can thus be used for approximating the optimal $\delta$
value up to a desired precision:
 for example, for point 4 the optimal $\delta$ is between $0.025$ and
 $0.05$.
It is expected that different input points will have different
local robustness, and the acceptable thresholds will thus need to be
set individually.

\begin{table}[t]
\centering
\caption{
Local adversarial robustness tests. All times are in seconds.
}
\scalebox{0.98}{
\begin{tabular}[htp]{cc|clrclrclrclrclrc|cr}
  \toprule
  &&&   
  \multicolumn{2}{c}{$\delta = 0.1$} 
  &&  
  \multicolumn{2}{c}{$\delta = 0.075$} 
  &&  
  \multicolumn{2}{c}{$\delta = 0.05$} 
  &&  
  \multicolumn{2}{c}{$\delta = 0.025$} 
  &&  
  \multicolumn{2}{c}{$\delta = 0.01$} 
  &&& 
  \multicolumn{1}{c}{Total}
  \\
  \cline{4-5} \cline{7-8} \cline{10-11} \cline{13-14} \cline{16-17}
  &&&   
      Result & Time
  &&  
      Result & Time
  &&  
      Result & Time
  &&  
      Result & Time
  &&  
      Result & Time
  &&&  
  \multicolumn{1}{c}{Time}
  \\
  \midrule
  Point 1 &&& 
  \sat{} & 135 
  && 
  \sat{} & 239
  && 
  \sat{} & 24
  && \unsat{} & 609
  && \unsat{} & 57
  &&& 1064
  \\

  Point 2 &&& 
  \unsat{} & 5880
  && 
  \unsat{} & 1167
  && 
  \unsat{} & 285
  && \unsat{} & 57
  && \unsat{} & 5
  &&& 7394
  \\

  Point 3 &&& 
  \unsat{} & 863
  && 
  \unsat{} & 436
  && 
  \unsat{} & 99
  && \unsat{} & 53
  && \unsat{} & 1
  &&& 1452
  \\

  Point 4 &&& 
  \sat{} & 2   
  && 
  \sat{} & 977
  && 
  \sat{} & 1168 
  && \unsat{} & 656
  && \unsat{} & 7 
  &&& 2810
  \\

  Point 5 &&& 
  \unsat{} & 14560   
  && 
  \unsat{} & 4344
  && 
  \unsat{} & 1331
  && \unsat{} & 221
  && \unsat{} & 6
  &&& 20462
  \\
  \bottomrule
\end{tabular}%
}
\label{table:localRobustness}
\end{table}%

Finally, we mention an additional variant of adversarial robustness
which we term \emph{global adversarial robustness}, and which can also be
solved by Reluplex.
Whereas local adversarial robustness is measured for a specific
$\vec{x}$, global adversarial robustness
applies to all inputs simultaneously.
This is expressed by encoding two side-by-side copies of the DNN
in question, $N_1$ and $N_2$, operating on separate input variables
$\vec{x_1}$ and $\vec{x_2}$, respectively,
such that $\vec{x_2}$ represents an adversarial perturbation of
$\vec{x_1}$.
We can then check whether 
$\| \vec{x_1} - \vec{x_2}\|_{\infty}\leq \delta$ implies that the two copies of the DNN
produce similar outputs.
Formally, we require that 
if $N_1$ and $N_2$ assign output $a$
values $p_1$ and $p_2$ respectively, then $|p_1-p_2|\leq \epsilon$. 
If this holds for every output, we say that the network
is $\epsilon$-globally-robust.  Global
adversarial robustness is harder to prove than the local variant,
because encoding two copies of the network results in twice as many ReLU
nodes and because the problem is not restricted to a small input
domain.
 We were able to prove global adversarial robustness only on
small networks; improving the scalability of this technique is left
for future work.

\section{Related Work}
\label{sec:relatedWork}

In~\cite{PuTa10}, the authors propose an
approach for verifying properties of neural networks with
sigmoid activation functions. They replace the
activation functions with piecewise linear approximations thereof, and then
invoke black-box SMT solvers. When spurious
counter-examples are found,
the approximation is refined.
The authors highlight the difficulty in scaling-up
this technique, and are able to tackle only small networks with at
most 20 hidden nodes~\cite{PuTa12}.

The authors of~\cite{BaIoLaVyNoCr16} propose a technique for finding
local adversarial examples in DNNs with ReLUs. Given an input point $\vec{x}$,
they encode the problem as a linear program and invoke a black-box LP
solver. 
The activation function issue is circumvented by
considering a sufficiently small neighborhood of $\vec{x}$, in which
all ReLUs are fixed at the active or inactive state, 
 making the problem convex. Thus, it is unclear how to
address an $\vec{x}$ for which one or more ReLUs are
on the boundary between active and inactive states.
 In contrast, Reluplex can be used on input domains for which ReLUs can
 have more than one possible state.

In a recent paper~\cite{HuKwWaWu16}, the authors propose a method for proving
the local adversarial robustness of DNNs. 
For a specific
input point $\vec{x}$, the authors attempt to prove consistent labeling in a
neighborhood of $\vec{x}$ by means of discretization: they reduce
the infinite neighborhood into a finite set of points, and check that
the labeling of these points is consistent. This process is then propagated through the
network, layer by layer. While the technique is general in the sense
that it is not tailored for a specific activation function, the
discretization process means that any \unsat{} result only holds modulo
the assumption that the finite sets correctly represent their infinite domains.
In contrast, our technique can guarantee that there are no 
 irregularities hiding between the discrete points. 

Finally, in~\cite{JaGhKoGaScZaPl15}, the authors
 employ hybrid techniques to analyze an ACAS X controller given in
 lookup-table form,
 seeking to identify \emph{safe input regions} in which collisions
 cannot occur.
It will be interesting to combine our technique with
that of~\cite{JaGhKoGaScZaPl15}, in order to verify that following the
advisories provided by the DNNs indeed leads to collision avoidance.

\section{Conclusion and Next Steps}
\label{sec:conclusion}

We presented a novel decision algorithm for solving queries on
deep neural networks with ReLU activation functions. The
technique is based on extending the simplex algorithm to support
the non-convex ReLUs in a way that allows their inputs and outputs to be
temporarily inconsistent and
then fixed as the algorithm progresses.  To guarantee
termination, some ReLU connections may need to be split upon --- but in many
cases this is not required, resulting in an efficient solution. Our
success in verifying properties of the ACAS Xu networks
indicates that the technique holds much potential for verifying
real-world DNNs.

In the future, we plan to increase the technique's scalability.
 Apart from making engineering improvements to
our implementation, we plan to explore better
strategies for the application of the Reluplex rules,
and to employ advanced conflict analysis techniques for reducing the amount of
case-splitting required. Another direction is to provide better
soundness guarantees without harming performance, for example
by replaying floating-point solutions using 
 precise arithmetic~\cite{KiBaTi14}, or by producing
externally-checkable correctness proofs~\cite{KaBaTiReHa16}.
Finally, we plan to extend our approach to handle DNNs with additional
kinds of layers. We speculate that the mechanism we applied to ReLUs
can be applied to other piecewise linear layers, such as
max-pooling layers.
%

\subsubsection{Acknowledgements.} We thank 
Neal Suchy from the Federal Aviation
Administration, Lindsey Kuper from Intel 
and Tim King from Google for their valuable comments and support.
This work was partially supported by a
 grant from Intel.

{
\bibliographystyle{abbrv}

}

\renewcommand{\thesection}{\Roman{section}} 
\renewcommand{\thesubsection}{\Roman{subsection}} 
\setcounter{section}{0}

\newpage
\noindent
{\huge Appendix: Supplementary Material}

\section{Verifying Properties in DNNs with ReLUs is NP-Complete}
\label{appendix:npc}

Let $N$ be a DNN with ReLUs and
let $\varphi$ denote a property that is a conjunction of linear
constraints on the inputs $\vec{x}$ and outputs $\vec{y}$ of $N$, i.e. $\varphi =
\varphi_1(\vec{x})\wedge\varphi_2(\vec{y})$.
We say that $\varphi$ \emph{is satisfiable on} $N$ if 
there exists an assignment $\assignment$ for
the variables $\vec{x}$ and $\vec{y}$ such that $\assignment(\vec{y})$ is the result
of propagating $\assignment(\vec{x})$ through $N$ and $\assignment$ satisfies $\varphi$.

\begin{claim}
The problem of determining whether $\varphi$ is satisfiable on $N$ for a given DNN
$N$ and a property $\varphi$ is NP-complete.
\end{claim}
\begin{proof}
We first show that the problem is in NP.
A satisfiability witness is simply an assignment $\assignment(\vec{x})$ for the input
variables $\vec{x}$.
 This witness can be checked by feeding the values for the
input variables forward through the network, obtaining the assignment
$\assignment(\vec{y})$ for the output
values, and checking whether $\varphi_1(\vec{x})\wedge\varphi_2(\vec{y})$ holds under the
assignment $\assignment$.

Next, we show that the problem is NP-hard, using a reduction from the
3-SAT problem. We will show how any 3-SAT formula $\psi$ can be transformed
into a DNN with ReLUs $N$ and a property $\varphi$, such that $\varphi$ is
satisfiable on $N$ if and only
if $\psi$ is satisfiable.

Let $\psi = C_1\wedge C_2\wedge\ldots\wedge C_n$ denote
a 3-SAT formula over variable set $X = \{x_1,\ldots, x_k\}$, i.e. each $C_i$
is a disjunction of three literals $q_i^1 \vee q_i^2 \vee q_i^3$ where
the $q$'s are variables from $X$ or their negations. The question is to
determine whether there exists an assignment $a:X\rightarrow \{0,1\}$
that satisfies $\psi$, i.e. that satisfies all the clauses simultaneously.

For simplicity, we first show the construction assuming that the input
nodes take the discrete values $0$ or $1$. Later we will explain how
this limitation can be relaxed, so that the only limitation on the
input nodes is that they be in the range $[0,1]$.

We begin by introducing the \emph{disjunction gadget} which, given
nodes $q_1,q_2,q_3\in\{0,1\}$, outputs a node $y_i$ that is $1$ if
$q_1+ q_2+ q_3\geq 1$ and $0$ otherwise. The gadget is shown below for the case that
the  $q_i$ literals are all variables (i.e. not negations of variables):
\begin{figure}[H]
\centering
\begin{tikzpicture}[shorten >=1pt,->,draw=black!50, node distance=\layersep]
  \node[input neuron] (qi1)  {$q_i^1$};
  \node[input neuron, below = 0.5cm of qi1] (qi2)  {$q_i^2$};
  \node[input neuron, below = 0.5cm of qi2] (qi3)  {$q_i^3$};
  
  \node[hidden neuron, right = 1cm of qi2] (sumQ) {$t_i$};
  \coordinate (middle) at ($(qi2)!.5!(sumQ)$);
  \node[input neuron, below = 1.5cm of middle] (constantOne) {$1$};

  \path (qi1) edge[] node[above] {$-1$} (sumQ);
  \path (qi2) edge[] node[above] {$-1$} (sumQ);
  \path (qi3) edge[] node[above] {$-1$} (sumQ);
  \path (constantOne) edge[] node[above] {$1$} (sumQ);

  \node[hidden neuron, right = 1cm of sumQ] (sumQF) {};
  \draw[dashed] (sumQ) -- node[above] {ReLU} (sumQF);

  \node[output neuron, right = 1cm of sumQF] (result) {$y_i$};
  \path (sumQF) edge[] node[above] {$-1$} (result);
  \path (constantOne) edge[] node[above] {$1$} (result);
\end{tikzpicture}
\end{figure}
\noindent
The disjunction gadget 
can be regarded as calculating the expression
\[
y_i = 1 - \max{}(0, 1 - \sum_{j=1}^3q_i^j)
\]
If there is at least one input variable set to $1$, 
$y_i$ will be equal to $1$. If all
inputs are $0$, $y_i$ will be equal to $0$. The crux of this gadget is
 that the ReLU operator allows us to guarantee that
even if there are multiple inputs set to $1$, the output $y_i$ will
still be precisely $1$. 

In order to handle any negative literals $q_i^j\equiv\neg x_j$,
before feeding the literal into the disjunction gadget 
 we first use a \emph{negation gadget}:
\begin{figure}[H]
\centering
\begin{tikzpicture}[shorten >=1pt,->,draw=black!50, node distance=\layersep]
  \node[input neuron] (qi1)  {$x_j$};
  
  \node[hidden neuron, right = 1cm of qi1] (sumQ) {$q_i^j$};
  \node[input neuron, below = 0.5cm of qi1] (constantOne) {$1$};

  \path (qi1) edge[] node[above] {$-1$} (sumQ);
  \path (constantOne) edge[] node[above] {$1$} (sumQ);
\end{tikzpicture}
\end{figure}

\noindent
This gadget simply calculates $1-x_j$, and then we continue as before.

The last part of the construction involves a \emph{conjunction
  gadget}:
\begin{figure}[H]
\centering
\begin{tikzpicture}[shorten >=1pt,->,draw=black!50, node distance=\layersep]
  \node[output neuron] (y1) {$y_1$};
  \node[output neuron, below = 1.5cm of y1] (yn) {$y_n$};

  \node[] at ($(y1)!.5!(yn)$) (middle) {\vdots};

  \node[output neuron, right = 2.5cm of middle] (y) {$y$}; 

  \path (y1) edge[] node[above] {1} (y);
  \path (yn) edge[] node[above] {1} (y);
\end{tikzpicture}
\end{figure}
\noindent
Assuming all nodes $y_1,\ldots,y_n$ are in the domain $\{0,1\}$, we
require that node $y$ be in the range $[n,n]$. Clearly this holds only
if $y_i=1$ for all $i$.

Finally, in order to check whether all clauses $C_1,\ldots,C_n$ are
simultaneously satisfied, we construct a disjunction gadget for each of
the clauses (using negation gadgets for their inputs as needed), and
combine them using a conjunction gadget:
\begin{figure}[H]
\centering
\begin{tikzpicture}[shorten >=1pt,->,draw=black!50, node distance=\layersep]
  \node[input neuron] (x1)  {$x_1$};
  \node[input neuron, below = 0.5cm of x1] (x2)  {$x_2$};
  \node[below = 0.5cm of x2] (inputVdots)  {$\vdots$};
  \node[input neuron, below = 0.5cm of inputVdots] (xn)  {$x_n$};
  \node[input neuron, below = 0.5cm of xn] (constantOneInput)  {$1$};

  \coordinate (middleInput) at ($(x1)!.5!(constantOneInput)$);
  \coordinate (rightOfInputs) at ($(middleInput) + (2.5cm,0)$);

  \node[right = 2.5cm of rightOfInputs] () {$\vdots$};  

 
  \node[hidden neuron, above = 1.5cm of rightOfInputs] (sumQ1) {$t_1$};
  \node[hidden neuron, right = 1cm of sumQ1] (sumQF1) {};
  \draw[dashed] (sumQ1) -- node[above]{ReLU} (sumQF1);
  \node[output neuron, right = 1cm of sumQF1] (y1) {$y_1$};
  \path (sumQF1) edge[] node[above] {$-1$} (y1);

  \node[input neuron, below = 0.5cm of sumQF1] (constantOne1) {$1$};
  \path (constantOne1) edge[] node[above] {$1$} (y1);


  \node[hidden neuron, below = 1.0cm of rightOfInputs] (sumQn) {$t_n$};
  \node[hidden neuron, right = 1cm of sumQn] (sumQFn) {};
  \draw[dashed] (sumQn) -- node[above]{ReLU} (sumQFn);
  \node[output neuron, right = 1cm of sumQFn] (yn) {$y_n$};
  \path (sumQFn) edge[] node[above] {$-1$} (yn);

  \node[input neuron, below = 0.5cm of sumQFn] (constantOnen) {$1$};
  \path (constantOnen) edge[] node[above] {$1$} (yn);


  \node[output neuron, right = 5cm of rightOfInputs] (y) {$y$};  

  \path (y1) edge[] node[above] {1} (y);
  \path (yn) edge[] node[above] {1} (y);


  \draw[rounded corners]
  ($(middleInput) + (-0.7cm, -2.8cm ) $)
  rectangle 
  ($(middleInput) + (0.7cm, 2.8cm ) $);
  
  \path ($(middleInput)+(0.7cm,0.0cm)$) 
  edge[] node[above] {} (sumQ1);

  \path ($(middleInput)+(0.7cm,1.0cm)$)
  edge[] node[above] {} (sumQn);

  \path ($(middleInput)+(0.7cm,-1.0cm)$)
  edge[] node[above] {} (sumQn);

  \path ($(middleInput)+(0.7cm,-2.0cm)$)
  edge[] node[above] {} (sumQ1);

  \path ($(middleInput)+(0.7cm,-2.0cm)$)
  edge[] node[above] {} (sumQn);

  \path ($(middleInput)+(0.7cm,2.0cm)$)
  edge[] node[above] {} (sumQ1);
\end{tikzpicture}
\end{figure}
\noindent
where the input variables are mapped to each $t_i$ node  according to the
definition of clause $C_i$.
 As we discussed before, node $y_i$ will be equal
to $1$ if clause $C_i$ is satisfied, and will be $0$ otherwise.
Therefore, node $y$ will be in the range
$[n,n]$  if and only if all clauses are simultaneously
satisfied.
 Consequently, an input assignment $a:X\rightarrow\{0,1\}$
satisfies the input and output constraints on the network if and only if it also satisfies the original
$\psi$, as needed.

The construction above is based on the assumption that we can require
that the input nodes take values in the discrete set $\{0,1\}$, which
does not fit our assumption that $\varphi_1(\vec{x})$ is a
conjunction of linear constraints. We show now how this requirement
can be relaxed.

Let $\epsilon>0$ be a very small number.
We set the
input range for each variable $x_i$ to be $[0,1]$, but we will ensure that
any feasible solution has $x_i\in[0,\epsilon]$ or $x_i\in[1-\epsilon,
1]$. We do this by adding to the network for each $x_i$ an auxiliary
gadget that uses ReLU nodes to compute the expression
\[
\max{}(0, \epsilon - x) + \max{}(0, x - 1 + \epsilon),
\]
and requiring that the output node of this gadget be in the range $[0, \epsilon]$. 
It is straightforward to show that this holds for $x\in [0,1]$ if and
only if $x\in[0,\epsilon]$ or $x\in[1-\epsilon,1]$.

The disjunction gadgets in our construction then change accordingly.
The $y_i$ nodes at the end of each
gadget will no longer take just the discrete values
$\{0,1\}$, but instead be in the range $[0,3\cdot \epsilon]$ if all
inputs were in the range $[0,\epsilon]$, or in the range
$[1-\epsilon,1]$ if at least one input was in the range
$[1-\epsilon,\epsilon]$.

If every input clause has at least one node in the range
$[1-\epsilon,1]$ then all $y_i$ nodes will be in the range
$[1-\epsilon,1]$, and consequently $y$ will be in the range
$[n(1-\epsilon), n]$. However, if at least one clause does not have a
node in the range $[1-\epsilon,1]$ then $y$ will be smaller than
$n(1-\epsilon)$ (for $\epsilon < \frac{1}{n+3}$).
Thus, by requiring that $y\in [n(1-\epsilon), n]$, the input and output
constraints will be satisfiable on the network if and only if $\psi$ is satisfiable; and the
satisfying assignment can be constructed by treating every
$x_i\in[0,\epsilon]$ as $0$ and every $x_i\in[1-\epsilon, 1]$ as $1$.
\qed
\end{proof}

\section{The Reluplex Calculus is Sound and Complete}
\label{appendix:soundness}

\newcommand{\derivation}{\mathcal{D}}
\newcommand{\derivationtree}{D}

We define a \emph{derivation tree} as a tree where
each node is a configuration whose children (if any) are obtained by
applying to it one of the derivation rules.  A derivation tree $\derivationtree$ \emph{derives} a
derivation tree $\derivationtree'$ if $\derivationtree'$ is obtained from $\derivationtree$ by applying exactly one
derivation rule to one of $\derivationtree$'s leaves.  A \emph{derivation} is a sequence
$\derivationtree_i$ of derivation trees such that $\derivationtree_0$
has only a single node and each $\derivationtree_i$ derives
$\derivationtree_{i+1}$. 
A \emph{refutation} is a derivation ending in a tree, all of whose leaves are
$\unsat{}$.  A \emph{witness} is a derivation ending in a tree, at least one of
whose leaves is $\sat{}$.
If $\phi$ is a conjunction of atoms, we say that $\derivation$ is a derivation
from $\phi$ if the initial tree in $\derivation$ contains the configuration
initialized from $\phi$.
A calculus is sound if, whenever a derivation $\derivation$ from $\phi$ is
either a refutation or a witness, $\phi$ is correspondingly unsatisfiable or satisfiable,
respectively.  A calculus is \emph{complete} if there always exists either a
refutation or a witness starting from any $\phi$.  

In order to prove that the Reluplex calculus is sound, we first prove the
following lemmas:

\begin{lemma}
\label{lemma:sameTableau}
\sloppy
Let $\derivation$ denote a derivation starting from a derivation tree $\derivationtree_0$ with a single node 
$s_0 = \langle \basic_0, T_0, \lb_0, \ub_0, \assignment_0, \reluSet_0 \rangle$.
Then, for every derivation tree $\derivationtree_i$ appearing in $\derivation$,
and for each node $s = \langle \basic, T, \lb, \ub, \assignment, \reluSet\rangle$
appearing in $\derivationtree_i$ (except for the distinguished nodes \sat{} and
\unsat{}), the following properties hold:
\begin{enumerate}[(i)]
\item an assignment satisfies $T_0$ if and only if it satisfies $T$; and
\item the assignment $\assignment$ satisfies $T$ (i.e., $\assignment$
  satisfies all equations in $T$).
\end{enumerate}
\end{lemma}
\begin{proof}
The proof is by induction on $i$.  For $i = 0$, the claim holds trivially
(recall that $\assignment_0$ assigns every variable to 0).
Now, suppose the claim holds for some $i$ and consider
$\derivationtree_{i+1}$.  $\derivationtree_{i+1}$ is equivalent to
$\derivationtree_i$ except for the addition of one or more nodes added by the
application of a single derivation rule $d$ to some node $s$ with tableau $T$.
Because $s$ appears in $\derivationtree_i$,
we know by the induction hypothesis that
an assignment satisfies $T_0$ iff it satisfies $T$ and that
$\assignment$ satisfies $T$.
Let $s'$ be a new node (not a distinguished node \sat{} or
\unsat{}) with tableau $T'$ and assignment $\assignment'$, introduced by the rule $d$.
Note that $d$ cannot be \reluSuccess{} or \failure{} as these introduce only
distinguished nodes, and that if $d$ is \learnLB{}, \learnUB{}, or \reluSplit{},
then both the tableau and the assigment are unchanged, so
both properties are trivially preserved.

Suppose $d$ is \pivot{1}, \pivot{2} or
\pivotForRelu{}. For any of these rules, $\assignment' = \assignment$ and
$T'=\pivotOperation{}(T,i,j)$ for some $i$ and $j$. Observe that by definition of
the \pivotOperation{} operation, the equations of $T$ logically entail those of
$T'$ and vice versa, and so they are satisfied by exactly the same
assignments.  From this observation, both properties follow easily.

The remaining cases are when $d$ is \update{}, \updateb{} or
\updatef{}. For these rules, $T' = T$, from which property (i) follows
trivially.  For property (ii), we first note that
$\assignment'=\updateOperation{}(\assignment,x_i,\delta)$ for some $i$ and
$\delta$. By definition of
the \updateOperation{} operation, because $\assignment$ satisfied the
equations of $T$, $\assignment'$ continues to satisfy these equations and so
(because $T' = T$) $\assignment'$ also satisfies $T'$.
\qed
\end{proof}

\begin{lemma}
\label{lemma:narrowingBounds}
\sloppy
Let $\derivation$ denote a derivation starting from a derivation tree
$\derivationtree_0$ with a single node  
$s_0 = \langle \basic_0, T_0, \lb_0, \ub_0, \assignment_0, \reluSet_0
\rangle$.
If there  exists an assignment $\assignment^*$ (not
necessarily $\assignment_0$) such that $\assignment^*$  satisfies $T_0$ and
$\lb_0(x_i)\leq \assignment^*(x_i)\leq \ub_0(x_i)$ for all $i$, then
for each derivation tree $\derivationtree_i$ appearing in $\derivation$
at least one of these two properties holds:
\begin{enumerate}[(i)]
\item $\derivationtree_i$ has  a \sat{} leaf.
\item $\derivationtree_i$ has  a  leaf
$s = \langle \basic, T, \lb, \ub, \assignment, \reluSet \rangle$ 
(that is not a distinguished node \sat{} or \unsat{})
such that $\lb(x_i)\leq \assignment^*(x_i)\leq \ub(x_i)$ for all $i$.
\end{enumerate}
\end{lemma}

\begin{proof}
The proof is again by induction on $i$. For $i=0$, property (ii)
holds trivially. 
Now, suppose the claim holds for some $i$ and consider
$\derivationtree_{i+1}$.  $\derivationtree_{i+1}$ is equivalent to
$\derivationtree_i$ except for the addition of one or more nodes added by the
application of a single derivation rule $d$ to a leaf $s$ of $\derivationtree_i$.

Due to the induction hypothesis, we know that $\derivationtree_i$ has
a leaf $\bar{s}$ that is either a \sat{} leaf or that satisfies property
(ii). If $\bar{s}\neq s$, then $\bar{s}$ also appears in
$\derivationtree_{i+1}$, and the claim holds. We will show that
the claim also holds when $\bar{s}=s$. 
Because none of the derivation rules can be applied to a \sat{} or
\unsat{} node, we know that node $s$ is not a distinguished \sat{} or
\unsat{} node, and we denote
$s = \langle \basic, T, \lb, \ub, \assignment, \reluSet \rangle$. 

If $d$ is \reluSuccess{}, $\derivationtree_{i+1}$ has a \sat{} leaf
and property (i) holds.
Suppose $d$ is \pivot{1}, \pivot{2}, \pivotForRelu{}, \update{},
\updateb{} or \updatef{}. In any of these cases, node $s$ has a single
child in $\derivationtree_{i+1}$, which we denote
$s' = \langle \basic', T', \lb', \ub', \assignment', \reluSet'
\rangle$. By definition of these derivation rules, $\lb'(x_j)=\lb(x_j)$ and
$\ub'(x_j)=\ub(x_j)$ for all $j$. Because node $s$ satisfies property
(ii), we get that 
$s'$ is a leaf that satisfies property (ii), as needed.

Suppose that $d$ is \reluSplit{}, applied to a pair $\langle
x_i,x_j\rangle\in R$. Node $s$ has two children in
$\derivationtree_{i+1}$: a state $s^+$ in which the lower bound for
$x_i$ is $0$, and a
state $s^-$ in which the upper bound for $x_i$ is $0$. All other lower and upper bounds in
$s^+$ and $s^-$ are identical to those of $s$. It is straightforward
to see that if $\assignment^*(x_i)\geq 0$ then property (ii) holds
for $s^+$, and if $\assignment^*(x_i)\leq 0$ then property (ii)
holds for $s^-$. Either way, $\derivationtree_{i+1}$ has a leaf for
which property (ii) holds, as needed.

\sloppy
Next, consider the case where  $d$ is \learnLB{}
 (the \learnUB{} case is symmetrical and is omitted).
Node $s$ has a single
child in $\derivationtree_{i+1}$, which we denote
$s' = \langle \basic', T', \lb', \ub', \assignment', \reluSet'
\rangle$. Let $x_i$ denote the variable to which $\learnLB{}$ was applied.
By definition, $\lb'(x_i)\geq \lb(x_i)$, and all other variable bounds
are unchanged between $s$ and $s'$.
Thus, it suffices to show that $\assignment^*(x_i)\geq
\lb'(x_i)$. Because $\assignment^*$ satisfies $T_0$, it follows from 
Lemma~\ref{lemma:sameTableau} that it satisfies $T$. By the induction
hypothesis, $\lb(x_j)\leq\assignment^*(x_j)\leq\ub(x_j)$ for all
$j$. The fact that $\assignment^*(x_i)\geq\lb'(x_i)$ then follows
directly from the guard condition of \learnLB{}.

The only remaining case is when $d$ is the \failure{} rule. We explain why
this case is impossible. Suppose towards contradiction that in node
$s$ the \failure{} rule is applicable to variable $x_i$, and suppose
(without loss of generality) that $\assignment(x_i) < \lb(x_i)$.
By the inductive hypothesis, we know that 
$\lb(x_j)\leq\assignment^*(x_j)\leq\ub(x_j)$ for all $j$, and by
Lemma~\ref{lemma:sameTableau} we know that $\assignment^*$ satisfies $T$.
Consequently, there must be a variable $x_k$ such that 
$(T_{i,k}>0\ \wedge\ \assignment(x_k)<\assignment^*(x_k))$, 
or
$(T_{i,k}<0\ \wedge\ \assignment(x_k)>\assignment^*(x_k))$. But because all
variables under $\assignment^*$ are within their bounds, this means that
$slack^+(x_i)\neq\emptyset$, which is contradictory to the fact that
the \failure{} rule was applicable in $s$.
\qed
\end{proof}

\begin{lemma}
\label{lemma:boundsAndRelus}
\sloppy
Let $\derivation$ denote a derivation starting from a derivation tree $\derivationtree_0$ with a single node 
$s_0 = \langle \basic_0, T_0, \lb_0, \ub_0, \assignment_0, \reluSet_0 \rangle$.
Then, for every derivation tree $\derivationtree_i$ appearing in $\derivation$,
and for each node $s = \langle \basic, T, \lb, \ub, \assignment, \reluSet\rangle$
appearing in $\derivationtree_i$ (except for the distinguished nodes \sat{} and
\unsat{}), the following properties hold:
\begin{enumerate}[(i)]
\item $\reluSet = \reluSet_0$; and
\item $\lb(x_i)\geq \lb_0(x_i)$ and 
$\ub(x_i)\leq \ub_0(x_i)$ for all $i$.
\end{enumerate}
\end{lemma}
\begin{proof}
Property (i) follows from the fact that none of the derivation rules
(except for \reluSuccess{} and \failure{}) changes the set
$R$. Property (ii) follows from the fact that the only rules 
(except for \reluSuccess{} and \failure{})
that update lower and upper variable bounds are \learnLB{} and
\learnUB{}, respectively, and that these rules can only increase lower
bounds or decrease upper bounds.
\end{proof}

\noindent
We are now ready to prove that the Reluplex calculus is sound and complete.
\begin{claim}
The Reluplex calculus is sound.
\end{claim}
\begin{proof}
  We begin with the satisfiable case. Let
  $\derivation{}$ denote a witness for
  $\phi$. By definition, the final tree $\derivationtree{}$ in
  $\derivation{}$ has a $\sat{}$
  leaf. Let 
$s_0 = \langle \basic_0, T_0, \lb_0, \ub_0, \assignment_0, \reluSet_0 \rangle$ 
 denote the initial state of $\derivationtree_0$ and let
  $s = \langle \basic, T, \lb, \ub, \assignment, \reluSet \rangle$ 
  denote a state in $\derivationtree$ in which the \reluSuccess{}
  rule was applied (i.e., a predecessor of a \sat{} leaf). 

  By Lemma~\ref{lemma:sameTableau}, $\assignment$ satisfies
  $T_0$. Also,
  by the guard conditions of the $\reluSuccess{}$ rule, 
  $\lb(x_i)\leq \assignment(x_i)\leq \ub(x_i)$ for all $i$. 
  By property (ii) of Lemma~\ref{lemma:boundsAndRelus}, this implies that
  $\lb_0(x_i)\leq \assignment(x_i)\leq \ub_0(x_i)$ for all $i$.
  Consequently, $\assignment$ satisfies every linear inequality in $\phi$.

  Finally, we observe that by the conditions of
  the \reluSuccess{} rule, $\assignment$ satisfies all ReLU
  constraints of $s$. From property (i) of
  Lemma~\ref{lemma:boundsAndRelus}, it follows that $\assignment$ also
  satisfies the ReLU constraints of $s_0$, which are precisely the ReLU
  constraints in $\phi$.
  We conclude that $\assignment$ satisfies every constraint in $\phi$,
  and hence $\phi$ is satisfiable, as needed.

  For the unsatisfiable case, it suffices to show that if $\phi$ is
  satisfiable then there cannot exist a refutation for it.
  This is a direct result of
  Lemma~\ref{lemma:narrowingBounds}: if $\phi$ is satisfiable, then
  there exists an assignment $\assignment^*$ that satisfies the
  initial tableau $T_0$, and for which all variables are within
  bounds. Hence, Lemma~\ref{lemma:narrowingBounds} implies that
  any derivation tree in any derivation $\derivation$ from $\phi$ must have a leaf that is not the
  distinguished \unsat{} leaf. It follows that there cannot exist a
  refutation for $\phi$.
  \qed
\end{proof}

\begin{claim}
The Reluplex calculus is complete.
\end{claim}
\begin{proof}
Having shown that the Reluplex calculus is sound, it suffices to show
a strategy for deriving a witness or a refutation 
for every 
 $\phi$ within a finite number of steps.
As mentioned in Section~\ref{sec:reluplex}, one such strategy involves two steps:
\begin{inparaenum}[(i)]
\item 
Eagerly apply the \reluSplit{} rule until it no longer applies; 
and
\item
 For every leaf of the resulting derivation tree, apply the simplex rules 
\pivot{1}, \pivot{2},
\update{}, and \failure{}, and the Reluplex rule \reluSuccess{}, in a way that guarantees a \sat{}
or an \unsat{} configuration is reached within a finite number of
steps.
\end{inparaenum}

Let $\derivationtree$  denote the derivation tree obtained after
step (i). In every leaf $s$ of 
$\derivationtree$,
all ReLU connections have been eliminated, meaning that the variable
bounds force each ReLU connection to be either active or inactive. 
 This means that
every such $s$ can be regarded as a pure
simplex problem, and that any solution to that simplex problem is
guaranteed to satisfy also the ReLU constraints in $s$.

The existence of a terminating simplex strategy for
deciding the satisfiability of each leaf of $\derivationtree$
 follows from the completeness of the simplex calculus~\cite{Va96}. One such
widely used strategy is \emph{Bland's Rule}~\cite{Va96}. We observe
that although the simplex \success{} rule does not exist in Reluplex,
it can be directly substituted with the \reluSuccess{} rule. This is
so because, having applied the \reluSplit{} rule to completion, any assignment that satisfies
the variable bounds in $s$ also satisfies the ReLU constraints in $s$.

It follows that for every
$\phi$ we can produce a witness or a refutation, as needed.
\qed
\end{proof}

\section{A Reluplex Strategy that Guarantees Termination}
\label{appendix:termination}
As discussed in Section~\ref{sec:evaluation}, our strategy for
applying the Reluplex rules was
to repeatedly fix any out-of-bounds violations first (using the
original simplex rules),
and only afterwards to correct any violated ReLU constraints using the
\updateb{}, \updatef{} and \pivotForRelu{} rules. If correcting a
violated ReLU constraint introduced new out-of-bounds violations, these
were again fixed using the simplex rules, and so on.

As mentioned above, there exist well known strategies for applying the simplex rules in a
way that guarantees that within a finite number of steps, either all variables become assigned to
values within their bounds, or the \failure{} rule is applicable (and
is applied)~\cite{Va96}. By using such a strategy for fixing
out-of-bounds violations, and by splitting on a ReLU pair whenever the
\updateb{}, \updatef{} or \pivotForRelu{} rules are applied to it more some
fixed number of times, termination is guaranteed.


\section{Under-Approximations}
\label{appendix:approximation}
Under-approximation can be integrated into the Reluplex algorithm in a straightforward
manner.
Consider a variable $x$ with lower and upper bounds $l(x)$ and
$u(x)$, respectively. Since we are searching for feasible solutions for
which  $x\in [l(x),u(x)]$, an under-approximation 
can be obtained by restricting this range, and only considering
feasible solutions for which $x\in
[l(x)+\epsilon,u(x)-\epsilon]$ for some small $\epsilon>0$.

Applying under-approximations can be particularly useful
when it effectively eliminates a ReLU constraint (consequently reducing the
potential number of case splits needed).
 Specifically, observe a ReLU pair
$x^f=\relu(x^b)$ for which we have 
 $\lb{}(x^b)\geq -\epsilon$ for a very small positive $\epsilon$.
We can under-approximate this range and
instead set $\lb{}(x^b)= 0$; and, as previously discussed, we can then fix
the ReLU pair to the active state. Symmetrical measures
can be employed when learning a very small upper bound for $x^f$, 
 in this case leading to the ReLU pair being fixed in the inactive state.

 Any feasible
solution that is found using this kind of under-approximation will be a feasible
solution for the original problem. However, if we determine that the
under-approximated problem is infeasible, the original may yet be
feasible.

\section{Encoding ReLUs for SMT and LP Solvers}
\label{appendix:encoding}
We demonstrate the encoding of ReLU nodes that we used for the
evaluation conducted using SMT and LP solvers. Let $y=\relu{}(x)$. In
the SMTLIB format, used by all SMT solvers that we tested, ReLUs were
encoded using an if-then-else construct: 

\begin{verbatim}
(assert (= y (ite (>= x 0) x 0)))
\end{verbatim}

\newcommand{\bon}{b\(\sb{\text{on}}\)}
\newcommand{\boff}{b\(\sb{\text{off}}\)}

In LP format this was encoded using mixed integer programming. Using
Gurobi's built-in Boolean type, we defined for every ReLU connection
a pair of Boolean variables, 
\bon{} and \boff{}, and used them to encode the two possible states of the connection.
Taking $M$ to be a very large positive constant, we used the following assertions:

\begin{alltt}
\bon{} + \boff{} = 1
y >= 0
x - y - M*\boff <= 0
x - y + M*\boff >= 0
y - M*\bon <= 0
x - M*\bon <= 0
\end{alltt}

When \bon{}$=1$ and \boff{}$=0$, the ReLU connection is in the active
state; and otherwise, when \bon{}$=0$ and \boff{}$=1$, it is in the inactive state.

 In the active case, because \boff{} $= 0$ the third and fourth
equations imply that $x=y$ (observe that $y$ is always non-negative). 
$M$ is very large, and can be regarded
as $\infty$; hence, because
\bon{}$=1$,  the last two equations merely imply that $x,y\leq\infty$,
and so pose no restriction on the solution.

In the inactive case, \bon{} $=0$, and so the last two equations force
$y=0$ and $x\leq 0$. In this case \boff{}$=1$ and so the third and fourth
 equations pose no restriction on the solution.

\section{Formal Definitions for Properties $\phi_1$,\ldots,$\phi_{10}$}
\label{appendix:properties}
The units for the ACAS Xu DNNs' inputs are:
\begin{itemize}
\item $\rho$: feet.
\item $\theta,\psi$: radians. 
\item $v_\text{own}, v_\text{int}$: feet per second.
\item $\tau$: seconds.
\end{itemize}
$\theta$ and $\psi$ are measured counter clockwise, and are always in
the range $[-\pi,\pi]$.

In line with the discussion in Section~\ref{sec:acasxu}, the family of 45 ACAS Xu
DNNs are indexed according to the previous action $a_\text{prev}$ and
time until loss of vertical separation $\tau$. The possible values are
for these two indices are:

\begin{enumerate}
\item $a_\text{prev}$: $[$Clear-of-Conflict, weak left, weak right,
  strong left, strong right$]$.
\item $\tau$: $[0,1,5,10,20,40,60,80,100]$.
\end{enumerate}

We use $N_{x,y}$ to denote the network trained for the
$x$-th value of $a_\text{prev}$ and $y$-th value of $\tau$. For
example, $N_{2,3}$ is the network trained for the case where
$a_\text{prev}=$  weak left
and
 $\tau=5$.
Using this notation, we now give the formal definition of each of the
properties $\phi_1,\ldots,\phi_{10}$ that we tested.

\subsubsection{Property $\phi_1$.}
\begin{itemize}
\item Description:
  If the intruder is distant and is
  significantly slower than the ownship, the score of a COC advisory will
  always be below a certain fixed threshold.
\item Tested on: all 45 networks.
\item Input constraints:
  $\rho\geq 55947.691$,
  $v_\text{own}\geq 1145$,
  $v_\text{int}\leq 60$.
\item Desired output property: the score for COC is at most $1500$.
\end{itemize}

\subsubsection{Property $\phi_2$.}
\begin{itemize}
\item Description:
  If the intruder is distant and is
  significantly slower than the ownship, the score of a COC advisory will
  never be maximal.
\item Tested on: $N_{x,y}$ for all $x\geq 2$ and for all $y$.
\item Input constraints:
  $\rho\geq 55947.691$,
  $v_\text{own}\geq 1145$,
  $v_\text{int}\leq 60$.
\item Desired output property: the score for COC is not the maximal score.
\end{itemize}

\subsubsection{Property $\phi_3$.}
\begin{itemize}
\item Description:
  If the intruder is directly ahead and is moving towards the ownship,
  the score for COC will not be minimal.
\item Tested on: all networks except $N_{1,7}$, $N_{1,8}$, and $N_{1,9}$.
\item Input constraints:  
    $1500 \leq \rho \leq 1800$,
    $-0.06 \leq \theta \leq 0.06$,
    $\psi \geq 3.10$,
    $v_\text{own}\geq 980$,
    $v_\text{int}\geq 960$.
\item Desired output property: the score for COC is not the minimal score.
\end{itemize}

\subsubsection{Property $\phi_4$.}
\begin{itemize}
\item Description:
  If the intruder is directly ahead and is moving away from the
  ownship but at a lower speed than that of the ownship,
  the score for COC will not be minimal.
\item Tested on: all networks except $N_{1,7}$, $N_{1,8}$, and $N_{1,9}$.
\item Input constraints:  
    $1500 \leq \rho \leq 1800$,
    $-0.06 \leq \theta \leq 0.06$,
    $\psi = 0$,
    $v_\text{own}\geq 1000$,
    $700 \leq v_\text{int}\leq 800$.
\item Desired output property: the score for COC is not the minimal score.
\end{itemize}

\subsubsection{Property $\phi_5$.}
\begin{itemize}
\item Description:
  If the intruder is near and approaching from the
  left, the network advises ``strong right''.
\item Tested on: $N_{1,1}$.
\item Input constraints:
    $250 \leq \rho \leq 400$,
    $0.2 \leq \theta \leq 0.4$,
    $-3.141592 \leq \psi \leq -3.141592 + 0.005$,
    $100 \leq v_\text{own}\leq 400$,
    $0 \leq v_\text{int}\leq 400$.
\item Desired output property: the score for ``strong right'' is the minimal score.
\end{itemize}

\subsubsection{Property $\phi_6$.}
\begin{itemize}
\item Description:
  If the intruder is sufficiently far away,
  the network advises COC.
\item Tested on: $N_{1,1}$.
\item Input constraints:
    $12000 \leq \rho \leq 62000$,
    $(0.7 \leq \theta \leq 3.141592)
    \vee
    (-3.141592 \leq \theta \leq -0.7)$,
    $-3.141592 \leq \psi \leq -3.141592 + 0.005$,
    $100 \leq v_\text{own}\leq 1200$,
    $0 \leq v_\text{int}\leq 1200$.
\item Desired output property: the score for COC is the minimal score.
\end{itemize}

\subsubsection{Property $\phi_7$.}
\begin{itemize}
\item Description:
  If vertical separation is large,
  the network will never advise a strong turn.
\item Tested on: $N_{1,9}$.
\item Input constraints:
    $0 \leq \rho \leq 60760$,
    $-3.141592 \leq \theta \leq 3.141592$,
    $-3.141592 \leq \psi \leq 3.141592$,
    $100 \leq v_\text{own}\leq 1200$,
    $0 \leq v_\text{int}\leq 1200$.
\item Desired output property: the scores for ``strong right'' and
  ``strong left'' are never the minimal scores.
\end{itemize}

\subsubsection{Property $\phi_8$.}
\begin{itemize}
\item Description:
 For a large vertical separation and a previous
 ``weak left'' advisory, the network will either output COC or
 continue advising ``weak left''.
\item Tested on: $N_{2,9}$.
\item Input constraints:
    $0 \leq \rho \leq 60760$,
    $-3.141592 \leq \theta \leq -0.75\cdot 3.141592$,
    $-0.1 \leq \psi \leq 0.1$,
    $600 \leq v_\text{own}\leq 1200$,
    $600 \leq v_\text{int}\leq 1200$.
\item
Desired output property: the score for ``weak left'' is minimal or the
score for COC is minimal.
\end{itemize}

\subsubsection{Property $\phi_9$.}
\begin{itemize}
\item Description:
Even if the previous advisory was ``weak right'',
the presence of a nearby intruder will cause the network to output a
 ``strong left'' advisory instead.
\item Tested on: $N_{3,3}$.
\item Input constraints:
    $2000 \leq \rho \leq 7000$,
    $-0.4 \leq \theta \leq -0.14$,
    $-3.141592 \leq \psi \leq -3.141592+0.01$,
    $100 \leq v_\text{own}\leq 150$,
    $0 \leq v_\text{int}\leq 150$.
\item Desired output property: the score for ``strong left'' is minimal.
\end{itemize}

\subsubsection{Property $\phi_{10}$.}
\begin{itemize}
\item Description:
For a far away intruder, the network advises COC.
\item Tested on: $N_{4,5}$.
\item Input constraints:
    $36000 \leq \rho \leq 60760$,
    $0.7 \leq \theta \leq 3.141592$,
    $-3.141592 \leq \psi \leq -3.141592+0.01$,
    $900 \leq v_\text{own}\leq 1200$,
    $600 \leq v_\text{int}\leq 1200$.
\item Desired output property: the score for COC is minimal.
\end{itemize}

\end{document}